%% file: iclr2020_conference.tex
\documentclass{article} 
\usepackage{iclr2020_conference,times}

\input{math_commands.tex}

\usepackage[utf8]{inputenc} 
\usepackage[T1]{fontenc}    
\usepackage[hidelinks]{hyperref}       
\usepackage{url}            
\usepackage{booktabs}       
\usepackage{amsfonts}       
\usepackage{nicefrac}       
\usepackage{microtype}      
\usepackage{amsmath,amssymb,amsfonts}
\usepackage{enumitem}
\usepackage{graphicx}
\usepackage{subfigure}

\usepackage{amsthm}
\newtheorem{theorem}{Theorem}
\newtheorem{lemma}[theorem]{Lemma}

\usepackage[colorinlistoftodos, textsize=tiny]{todonotes}

\usepackage[toc]{glossaries}
\glsdisablehyper
\input{acronyms}

\title{Sharing Knowledge in Multi-Task \\ Deep Reinforcement Learning}

\author{Carlo D'Eramo \& Davide Tateo \\
Department of Computer Science\\
TU Darmstadt, IAS\\
Hochschulstraße 10, 64289, Darmstadt, Germany \\
\texttt{\{carlo.deramo,davide.tateo\}@tu-darmstadt.de} \\
\And
Andrea Bonarini \& Marcello Restelli \\
Politecnico di Milano, DEIB \\
Piazza Leonardo da Vinci 32, 20133, Milano \\
\texttt{\{andrea.bonarini,marcello.restelli\}@polimi.it} \\
\AND
Jan Peters \\
TU Darmstadt, IAS\\
Hochschulstraße 10, 64289, Darmstadt, Germany \\
Max Planck Institute for Intelligent Systems \\
Max-Planck-Ring 4, 72076, T\"{u}bingen, Germany \\
\texttt{jan.peters@tu-darmstadt.de}
}

\iclrfinalcopy 
\begin{document}

\maketitle

\begin{abstract}
We study the benefit of sharing representations among tasks to enable the effective use of deep neural networks in Multi-Task Reinforcement Learning. We leverage the assumption that learning from different tasks, sharing common properties, is helpful to generalize the knowledge of them resulting in a more effective feature extraction compared to learning a single task. Intuitively, the resulting set of features offers performance benefits when used by Reinforcement Learning algorithms. We prove this by providing theoretical guarantees that highlight the conditions for which is convenient to share representations among tasks, extending the well-known finite-time bounds of Approximate Value-Iteration to the multi-task setting. In addition, we complement our analysis by proposing multi-task extensions of three Reinforcement Learning algorithms that we empirically evaluate on widely used Reinforcement Learning benchmarks showing significant improvements over the single-task counterparts in terms of sample efficiency and performance.
\end{abstract}

\section{Introduction}\label{S:intro}
\gls{mtl} ambitiously aims to learn multiple tasks jointly instead of learning them separately, leveraging the assumption that the considered tasks have common properties which can be exploited by \gls{ml} models to generalize the learning of each of them. For instance, the features extracted in the hidden layers of a neural network trained on multiple tasks have the advantage of being a general representation of structures common to each other. This translates into an effective way of learning multiple tasks at the same time, but it can also improve the learning of each individual task compared to learning them separately~\citep{caruana1997multitask}. Furthermore, the learned representation can be used to perform \gls{tl}, i.e. using it as a preliminary knowledge to learn a new similar task resulting in a more effective and faster learning than learning the new task from scratch~\citep{baxter2000model, thrun2012learning}.

The same benefits of extraction and exploitation of common features among the tasks achieved in \gls{mtl}, can be obtained in \gls{mtrl} when training a single agent on multiple \gls{rl} problems with common structures~\citep{taylor2009transfer, lazaric2012transfer}. In particular, in \gls{mtrl} an agent can be trained on multiple tasks in the same domain, e.g. riding a bicycle or cycling while going towards a goal, or on different but similar domains, e.g. balancing a pendulum or balancing a double pendulum\footnote{For simplicity, in this paper we refer to the concepts of \textit{task} and \textit{domain} interchangeably.}. Considering recent advances in \gls{drl} and the resulting increase in the complexity of experimental benchmarks, the use of \gls{dl} models, e.g. deep neural networks, has become a popular and effective way to extract common features among tasks in \gls{mtrl} algorithms~\citep{rusu2015policy, liu2016decoding, higgins2017darla}. However, despite the high representational capacity of \gls{dl} models, the extraction of good features remains challenging. For instance, the performance of the learning process can degrade when unrelated tasks are used together~\citep{caruana1997multitask, baxter2000model}; another detrimental issue may occur when the training of a single model is not balanced properly among multiple tasks~\citep{hessel2018multi}.

Recent developments in \gls{mtrl} achieve significant results in feature extraction by means of algorithms specifically developed to address these issues. While some of these works rely on a single deep neural network to model the multi-task agent~\citep{liu2016decoding, yang2017multi, hessel2018multi,wulfmeier2019regularized}, others use multiple deep neural networks, e.g. one for each task and another for the multi-task agent~\citep{rusu2015policy, parisotto2015actor, higgins2017darla, teh2017distral}. Intuitively, achieving good results in \gls{mtrl} with a single deep neural network is more desirable than using many of them, since the training time is likely much less and the whole architecture is easier to implement. In this paper we study the benefits of shared representations among tasks. We theoretically motivate the intuitive effectiveness of our method, deriving theoretical guarantees that exploit the theoretical framework provided by~\cite{maurer2016benefit}, in which the authors present upper bounds on the quality of learning in \gls{mtl} when extracting features for multiple tasks in a single shared representation. The significancy of this result is that the cost of learning the shared representation decreases with a factor $\mathcal{O}(\nicefrac{1}{\sqrt{T}})$, where $T$ is the number of tasks for many function approximator hypothesis classes. The main \textit{contribution} of this work is twofold.
\begin{enumerate}
\item We derive upper confidence bounds for \gls{avi} and \gls{api}\footnote{All proofs and the theorem for \gls{api} are in Appendix~\ref{A:api_bound}.}~\citep{farahmand2011regularization} in the \gls{mtrl} setting, and we extend the approximation error bounds in~\cite{maurer2016benefit} to the case of multiple tasks with different dimensionalities. Then, we show how to combine these results resulting in, to the best of our knowledge, the first proposed extension of the finite-time bounds of \gls{avi}/\gls{api} to \gls{mtrl}. Despite being an extension of previous works, we derive these results to justify our approach showing how the error propagation in \gls{avi}/\gls{api} can theoretically benefit from learning multiple tasks jointly.
\item We leverage these results proposing a neural network architecture, for which these bounds hold with minor assumptions, that allow us to learn multiple tasks with a single regressor extracting a common representation. We show an empirical evidence of the consequence of our bounds by means of a variant of \gls{fqi}~\citep{ernst2005tree}, based on our shared network and for which our bounds apply, that we call \gls{mfqi}. Then, we perform an empirical evaluation in challenging \gls{rl} problems proposing multi-task variants of the \gls{dqn}~\citep{mnih2015human} and \gls{ddpg}~\citep{lillicrap2015continuous} algorithms. These algorithms are practical implementations of the more general \gls{avi}/\gls{api} framework, designed to solve complex problems. In this case, the bounds apply to these algorithms only with some assumptions, e.g. stationary sampling distribution. The outcome of the empirical analysis joins the theoretical results, showing significant performance improvements compared to the single-task version of the algorithms in various \gls{rl} problems, including several MuJoCo~\citep{todorov2012mujoco} domains.
\end{enumerate}

\section{Preliminaries}
Let $B(\mathcal{X})$ be the space of bounded measurable functions w.r.t. the $\sigma$-algebra $\sigma_{\mathcal{X}}$, and similarly $B(\mathcal{X}, L)$ be the same bounded by $L < \infty$.

A \gls{mdp} is defined as a $5$-tuple $\mathcal{M} = <\mathcal{S}, \mathcal{A}, \mathcal{P}, \mathcal{R}, \gamma>$, where $\mathcal{S}$ is the state space, $\mathcal{A}$ is the action space, $\mathcal{P}: \mathcal{S} \times \mathcal{A} \to \mathcal{S}$ is the transition distribution where $\mathcal{P}(s' | s, a)$ is the probability of reaching state $s'$ when performing action $a$ in state $s$, $\mathcal{R}: \mathcal{S} \times \mathcal{A} \times \mathcal{S} \to \mathbb{R}$ is the reward function, and $\gamma \in (0,1]$ is the discount factor. A \textit{deterministic policy} $\pi$ maps, for each state, the action to perform: $\pi:\mathcal{S} \to \mathcal{A}$. Given a policy $\pi$, the value of an action $a$ in a state $s$ represents the expected discounted cumulative reward obtained by performing $a$ in $s$ and following $\pi$ thereafter: $Q^\pi(s,a) \triangleq \mathbb{E}[\sum^\infty_{k=0} \gamma^k r_{i+k+1}|s_i=s, a_i=a, \pi]$, where $r_{i+1}$ is the reward obtained after the $i$-th transition. The expected discounted cumulative reward is maximized by following the \textit{optimal} policy $\pi^*$ which is the one that determines the optimal action values, i.e., the ones that satisfy the Bellman optimality equation \citep{bellman1954theory}: $Q^*(s,a) \triangleq \int_{\mathcal{S}} \mathcal{P}(s'|s,a) \left[\mathcal{R}(s,a,s') + \gamma \max_{a'} Q^*(s',a')\right] ds'$. The solution of the Bellman optimality equation is the fixed point of the optimal Bellman operator $\mathcal{T}^* : B(\mathcal{S} \times \mathcal{A}) \to B(\mathcal{S} \times \mathcal{A})$ defined as $(\mathcal{T}^*Q)(s,a) \triangleq \int_\mathcal{S} \mathcal{P}(s'|s,a) [\mathcal{R}(s,a,s') + \gamma \max_{a'} Q(s',a')] ds'$.
In the \gls{mtrl} setting, there are multiple \glspl{mdp} $\mathcal{M}^{(t)} = <\mathcal{S}^{(t)}, \mathcal{A}^{(t)}, \mathcal{P}^{(t)}, \mathcal{R}^{(t)}, \gamma^{(t)}>$ where $t \in \lbrace 1, \dots, T \rbrace$ and $T$ is the number of \glspl{mdp}. For each \gls{mdp} $\mathcal{M}^{(t)}$, a deterministic policy $\pi_t : \mathcal{S}^{(t)} \to \mathcal{A}^{(t)}$ induces an action-value function $Q_t^{\pi_t}(s^{(t)}, a^{(t)}) = \mathbb{E}[\sum^\infty_{k=0} \gamma^k r^{(t)}_{i+k+1}|s_i=s^{(t)}, a_i=a^{(t)}, \pi_t]$. In this setting, the goal is to maximize the sum of the expected cumulative discounted reward of each task.

In our theoretical analysis of the \gls{mtrl} problem, the complexity of representation plays a central role. As done in~\citet{maurer2016benefit}, we consider the Gaussian complexity, a variant of the well-known Rademacher complexity, to measure the complexity of the representation. Given a set $\mathbf{\bar{X}} \in \mathcal{X}^{Tn}$ of $n$ input samples for each task $t \in \lbrace 1,\dots,T\rbrace$, and a class $\mathcal{H}$ composed of $k \in \lbrace 1,\dots,K\rbrace$ functions, the Gaussian complexity of a random set $\mathcal{H}(\mathbf{\bar{X}})=\left\{(h_k(X_{ti}))\, : h\in\mathcal{H}\right\}\subseteq \mathbb{R}^{KTn}$ is defined as follows:
\begin{equation}
 G(\mathcal{H}(\mathbf{\bar{X}}))=\mathbb{E}\left[\sup_{h\in\mathcal{H}}\sum_{tki}\gamma_{tki}h_k(X_{ti}) \middle| X_{ti} \right],
\end{equation}
where $\gamma_{tki}$ are independent standard normal variables. We also need to define the following quantity, taken from~\cite{maurer2016chain}: let $\boldsymbol{\gamma}$ be a vector of $m$ random standard normal variables, and $f\in\mathcal{F}: Y \to\mathbb{R}^m$, with $Y \subseteq \mathbb{R}^n$, we define
\begin{equation}
O(\mathcal{F}) = \sup_{y, y' \in Y, y \neq y'} \mathbb{E}\left[\sup_{f \in \mathcal{F}}\dfrac{\langle\boldsymbol{\gamma}, f(y) - f(y')\rangle}{\lVert y - y' \rVert}\right].\label{E:L_quotient}
\end{equation}
Equation~\ref{E:L_quotient} can be viewed as a Gaussian average of Lipschitz quotients, and appears in the bounds provided in this work.
Finally, we define $L(\mathcal{F})$ as the upper bound of the Lipschitz constant of all the functions $f$ in the function class $\mathcal{F}$.

\section{Theoretical analysis}\label{S:bounds}
The following theoretical study starts from the derivation of theoretical guarantees for \gls{mtrl} in the \gls{avi} framework, extending the results of~\cite{farahmand2011regularization} in the \gls{mtrl} scenario. Then, to bound the approximation error term in the \gls{avi} bound, we extend the result described in~\cite{maurer2006bounds} to \gls{mtrl}. As we discuss, the resulting bounds described in this section clearly show the benefit of sharing representation in \gls{mtrl}. To the best of our knowledge, this is the first general result for \gls{mtrl}; previous works have focused on finite \glspl{mdp}~\citep{brunskill2013sample} or linear models~\citep{lazaric2011transfer}.

\subsection{Multi-task representation learning}
The multi-task representation learning problem consists in learning simultaneously a set of $T$ tasks $\mu_t$, modeled as probability measures over the space of the possible input-output pairs $(x, y)$, with $x\in\mathcal{X}$ and $y\in\mathbb{R}$, being $\mathcal{X}$ the input space.
Let $w\in\mathcal{W}:\mathcal{X} \to \mathbb{R}^J$, $h\in\mathcal{H}:\mathbb{R}^J \to \mathbb{R}^K$ and $f\in\mathcal{F}:\mathbb{R}^K \to \mathbb{R}$ be functions chosen from their respective hypothesis classes. The functions in the hypothesis classes must be Lipschitz continuous functions.
Let $\mathbf{\bar{Z}}=(\mathbf{Z}_1,\dots,\mathbf{Z}_T)$ be the multi-sample over the set of tasks $\boldsymbol{\mu}=(\mu_1,\dots,\mu_T)$, where $\mathbf{Z}_t=(Z_{t1},\dots, Z_{tn})\sim\mu_t^n$ and $Z_{ti}=(X_{ti}, Y_{ti})\sim \mu_t$.
We can formalize our regression problem as the following minimization problem:
\begin{align}
 \min \left\{ \dfrac{1}{nT}\sum_{t=1}^{T}\sum_{i=1}^{N}\ell(f_t(h(w_t(X_{ti}))), Y_{ti}): \mathbf{f}\in\mathcal{F}^T, h\in\mathcal{H}, \mathbf{w}\in\mathcal{W}^T \right\},\label{E:mtrl}
\end{align}
where we use $\mathbf{f}=(f_1,\dots, f_T)$, $\mathbf{w}=(w_1,\dots, w_T)$, and define the minimizers of Equation~(\ref{E:mtrl}) as $\mathbf{\hat{w}}$, $\hat{h}$, and $\mathbf{\hat{f}}$.
We assume that the loss function $\ell: \mathbb{R}\times\mathbb{R}\to[0,1]$ is 1-Lipschitz in the first argument for every value of the second argument. While this assumption may seem restrictive, the result obtained can be easily scaled to the general case. To use the principal result of this section, for a generic loss function $\ell'$, it is possible to use $\ell(\cdot) = \nicefrac{\ell'(\cdot)}{\epsilon_{\text{max}}}$, where $\epsilon_{\text{max}}$ is the maximum value of $\ell'$.
The expected loss over the tasks, given $\mathbf{w}$, $h$ and $\mathbf{f}$ is the task-averaged risk:
\begin{equation}
 \varepsilon_{\text{avg}}(\mathbf{w}, h, \mathbf{f}) = \dfrac{1}{T} \sum_{t=1}^T \mathbb{E}\left[ \ell(f_t(h(w_t(X))), Y) \right]
\end{equation}
The minimum task-averaged risk, given the set of tasks $\boldsymbol{\mu}$ and the hypothesis classes $\mathcal{W}$, $\mathcal{H}$ and $\mathcal{F}$ is $\varepsilon_{\text{avg}}^*$, and the corresponding minimizers are $\mathbf{w}^*$, $h^*$ and $\mathbf{f}^*$. 

\subsection{Multi-task Approximate Value Iteration bound}
We start by considering the bound for the \gls{avi} framework which applies for the single-task scenario.

\begin{theorem}(Theorem 3.4 of~\cite{farahmand2011regularization})
  Let K be a positive integer, and $Q_{\text{max}} \leq \frac{R_{\text{max}}}{1-\gamma}$. Then for any sequence $(Q_k )^K_{k=0}\subset B(\mathcal{S}\times\mathcal{A}, Q_{\text{max}})$ and the corresponding sequence $(\varepsilon_k)_{k=0}^{K-1}$, where $\varepsilon_k=\lVert Q_{k+1} - \mathcal{T}^*Q_{k} \rVert^2_\nu$, we have:
  \begin{align}
    \lVert Q^* - Q^{\pi_K}\rVert_{1,\rho} \leq \dfrac{2\gamma}{(1-\gamma)^2}\left[\inf_{r \in [0,1]} C^{\frac{1}{2}}_{\text{VI},\rho,\nu}(K;r)\mathcal{E}^{\frac{1}{2}}(\varepsilon_{0}, \dots, \varepsilon_{K-1};r)+\dfrac{2}{1 - \gamma}\gamma^K R_{\text{max}}\right],\label{E:farahmand_avi}
  \end{align}
  where
  \begin{align}
     C_{\text{VI},\rho,\nu}(K;r) = \left(\dfrac{1-\gamma}{2}\right)^2\sup_{\pi'_1,\dots,\pi'_K}\sum^{K-1}_{k=0}a_k^{2(1-r)}
 &\left[\sum_{m \geq 0}\gamma^m\Big(c_{\text{VI}_1,\rho,\nu}(m,K-k;\pi'_K)\right.\nonumber\\
 &\left.+ c_{\text{VI}_2,\rho,\nu}(m+1;\pi'_{k+1},\dots,\pi'_{K})\Big)\vphantom{\sum_{m \geq 0}}\right]^2,
  \end{align}
  with $\mathcal{E}(\varepsilon_{0}, \dots, \varepsilon_{K-1};r)=\sum^{K-1}_{k=0}\alpha_k^{2r}\varepsilon_{k}$, the two coefficients $c_{\text{VI}_1,\rho,\nu}$, $c_{\text{VI}_2,\rho,\nu}$, the distributions $\rho$ and $\nu$, and the series $\alpha_k$ are defined as in~\cite{farahmand2011regularization}\label{T:farahmand_avi}.
\end{theorem}

In the multi-task scenario, let the average approximation error across tasks be:
\begin{equation}
 \varepsilon_{\text{avg},k}(\mathbf{\hat{w}}_k, \hat{h}_k, \mathbf{\hat{f}}_k) = \dfrac{1}{T}\sum_{t=1}^T\lVert Q_{t,k+1} - \mathcal{T}^*_tQ_{t,k} \rVert^2_\nu,
 \label{E:eps_avg_def}
\end{equation}
where $Q_{t,k+1} = \hat{f}_{t,k} \circ \hat{h}_{k} \circ \hat{w}_{t,k}$, and $\mathcal{T}^*_t$ is the optimal Bellman operator of task $t$.

In the following, we extend the \gls{avi} bound of Theorem~\ref{T:farahmand_avi} to the multi-task scenario, by computing the average loss across tasks and pushing inside the average using Jensen's inequality.

\begin{theorem}
Let K be a positive integer, and $Q_{\text{max}} \leq \frac{R_{\text{max}}}{1-\gamma}$. Then for any sequence $(Q_k )^K_{k=0}\subset B(\mathcal{S}\times\mathcal{A}, Q_{\text{max}})$ and the corresponding sequence $(\varepsilon_{\text{avg},k})_{k=0}^{K-1}$, where $\varepsilon_{\text{avg},k}=\dfrac{1}{T}\sum_{t=1}^T\lVert Q_{t,k+1} - \mathcal{T}^*_tQ_{t,k} \rVert^2_\nu$, we have:
 \begin{align}
  \dfrac{1}{T} \sum_{t=1}^T \lVert Q^*_t - Q_t^{\pi_K}\rVert_{1,\rho} \leq \dfrac{2\gamma}{(1-\gamma)^2} \left[ \inf_{r \in [0,1]} C_{\text{VI}}^{\frac{1}{2}}(K;r)\mathcal{E}_{\text{avg}}^{\frac{1}{2}}(\varepsilon_{\text{avg}, 0}, \dots, \varepsilon_{\text{avg}, K-1};r) + \dfrac{2 \gamma^K R_{\text{max},\text{avg}}}{1 - \gamma}\right]\label{E:avi_bound}
\end{align}
with $\mathcal{E}_{\text{avg}} = \sum^{K-1}_{k=0}\alpha_k^{2r}\varepsilon_{\text{avg},k}$, $\gamma = \underset{t \in \lbrace 1,\dots,T \rbrace}{\max} \gamma_t$, $C_{\text{VI}}^{\frac{1}{2}}(K;r) = \underset{t \in \lbrace 1, \dots, T \rbrace}{\max} C^{\frac{1}{2}}_{\text{VI},\rho,\nu}(K;t,r)$, $R_{\text{max},\text{avg}} = \dfrac{1}{T} \sum_{t=1}^T R_{\text{max},t}$ and $\alpha_k = \begin{cases} \frac{(1 - \gamma)\gamma^{K-k-1}}{1-\gamma^{K+1}} & 0 \leq k < K,\\ \frac{(1-\gamma)\gamma^K}{1 - \gamma^{K+1}} & k = K \end{cases}$.
\label{T:avi_bound}
\end{theorem}

\paragraph{Remarks} Theorem~\ref{T:avi_bound} retains most of the properties of Theorem 3.4 of~\cite{farahmand2011regularization}, except that the regression error in the bound is now task-averaged. Interestingly, the second term of the sum in Equation~(\ref{E:avi_bound}) depends on the average maximum reward for each task. In order to obtain this result we use an overly pessimistic bound on $\gamma$ and the concentrability coefficients, however this approximation is not too loose if the \glspl{mdp} are sufficiently similar.

\subsection{Multi-task approximation error bound}
We bound the task-averaged approximation error $\varepsilon_{\text{avg}}$ at each \gls{avi} iteration $k$ involved in~(\ref{E:avi_bound}) following a derivation similar to the one proposed by~\cite{maurer2016benefit}, obtaining:
\begin{theorem}
Let $\boldsymbol{\mu}$, $\mathcal{W}$, $\mathcal{H}$ and $\mathcal{F}$ be defined as above and assume $0 \in \mathcal{H}$ and $f(0) = 0, \forall f \in \mathcal{F}$. Then for $\delta > 0$ with probability at least $1 - \delta$ in the draw of $\bar{\mathbf{Z}} \sim \prod_{t=1}^T \mu_t^n$ we have that
\begin{align}
\varepsilon_{\text{avg}}(\mathbf{\hat{w}}, \hat{h}, \mathbf{\hat{f}}) \leq& \, L(\mathcal{F}) \left( c_1 \dfrac{L(\mathcal{H})\sup_{l\in\{1,\dots,T\}} G(\mathcal{W}(\mathbf{X}_l))}{n} + c_2\dfrac{\sup_\mathbf{w}\lVert \mathbf{w}(\bar{\mathbf{X}})\rVert O(\mathcal{H})}{nT}\right.\nonumber\\
&\left.+ c_3\dfrac{\min_{p \in P}G(\mathcal{H}(p))}{nT}\right)+ c_4 \dfrac{\sup_{h,\mathbf{w}}\lVert h(\mathbf{w}(\bar{\mathbf{X}}))\rVert O(\mathcal{F})}{n\sqrt{T}}
+ \sqrt{\frac{8\ln(\frac{3}{\delta})}{nT}} + \varepsilon^*_{\text{avg}}.\label{E:bound}
\end{align}\label{T:apprx}
\end{theorem}

\paragraph{Remarks}
The assumptions $0 \in \mathcal{H}$ and $f(0) = 0$ for all $f \in \mathcal{F}$ are not essential for the proof and are only needed to simplify the result. For reasonable function classes, the Gaussian complexity $G(\mathcal{W}(\mathbf{X}_l))$ is $\mathcal{O}(\sqrt n)$. If $\sup_\mathbf{w}\lVert \mathbf{w}(\bar{\mathbf{X}})\rVert$ and $\sup_{h,\mathbf{w}}\lVert h(\mathbf{w}(\bar{\mathbf{X}}))\rVert$ can be uniformly bounded, then they are $\mathcal{O}(\sqrt{nT})$. For some function classes, the Gaussian average of Lipschitz quotients $O(\cdot)$ can be bounded independently from the number of samples.
Given these assumptions, the first and the fourth term of the right hand side of Equation~(\ref{E:bound}), which represent respectively the cost of learning the meta-state space $\mathbf{w}$ and the task-specific $\mathbf{f}$ mappings, are both $\mathcal{O}(\nicefrac{1}{\sqrt n})$. The second term represents the cost of learning the multi-task representation $h$ and is $\mathcal{O}(\nicefrac{1}{\sqrt{nT}})$, thus vanishing in the multi-task limit $T \to \infty$. The third term can be removed if $\forall h \in \mathcal{H}, \exists p_0 \in P : h(p) = 0$; even when this assumption does not hold, this term can be ignored for many classes of interest, e.g. neural networks, as it can be arbitrarily small.

The last term to be bounded in~(\ref{E:bound}) is the minimum average approximation error $\varepsilon_{\text{avg}}^*$ at each \gls{avi} iteration $k$. Recalling that the task-averaged approximation error is defined as in~(\ref{E:eps_avg_def}), applying Theorem 5.3 by~\cite{farahmand2011regularization} we obtain:
\begin{lemma}
Let $Q_{t,k}^*, \forall t \in \lbrace 1, \dots, T \rbrace$ be the minimizers of $\varepsilon_{\text{avg},k}^*$, $\check{t}_k = \arg \max_{t \in \lbrace 1, \dots, T \rbrace} \lVert Q_{t,k+1}^* - \mathcal{T}^*_tQ_{t,k} \rVert^2_\nu$, and $b_{k,i} = \lVert Q_{\check{t}_k,i+1} - \mathcal{T}^*_{\check{t}}Q_{\check{t}_k,i}\rVert_\nu$, then:
\begin{align}
\varepsilon_{\text{avg},k}^* \leq &\left(\vphantom{\sum_{i=0}^{k-1}}\lVert Q_{\check{t}_k,k+1}^* - (\mathcal{T}^*_{\check{t}})^{k+1}Q_{\check{t}_k,0} \rVert_\nu 
+ \sum_{i=0}^{k-1}(\gamma_{\check{t}_k} C_{\text{AE}}(\nu;\check{t}_k, P))^{i+1}b_{k,k-1-i}\right)^2,\label{E:eps_star_bound}
\end{align}
with $C_{\text{AE}}$ defined as in~\cite{farahmand2011regularization}.
\label{T:eps_star}
\end{lemma}

\paragraph{Final remarks}
The bound for \gls{mtrl} is derived by composing the results in Theorems~\ref{T:avi_bound} and~\ref{T:apprx}, and Lemma~\ref{T:eps_star}.
The results above highlight the advantage of learning a shared representation. The bound in Theorem~\ref{T:avi_bound} shows that a small approximation error is critical to improve the convergence towards the optimal action-value function, and the bound in Theorem~\ref{T:apprx} shows that the cost of learning the shared representation at each \gls{avi} iteration is mitigated by using multiple tasks. This is particularly beneficial when the feature representation is complex, e.g. deep neural networks.

\subsection{Discussion}
As stated in the remarks of Equation~(\ref{E:bound}), the benefit of \gls{mtrl} is evinced by the second component of the bound, i.e. the cost of learning $h$, which vanishes with the increase of the number of tasks. Obviously, adding more tasks require the shared representation to be large enough to include all of them, undesirably causing the term $\sup_{h,\mathbf{w}}\lVert h(\mathbf{w}(\mathbf{\bar{X}}))\rVert$ in the fourth component of the bound to increase. This introduces a tradeoff between the number of features and number of tasks; however, for a reasonable number of tasks the number of features used in the single-task case is enough to handle them, as we show in some experiments in Section~\ref{S:exp}. Notably, since the \gls{avi}/\gls{api} framework provided by~\cite{farahmand2011regularization} provides an easy way to include the approximation error of a generic function approximator, it is easy to show the benefit in \gls{mtrl} of the bound in Equation~(\ref{E:bound}). Despite being just multi-task extensions of previous works, our results are the first one to theoretically show the benefit of sharing representation in \gls{mtrl}. Moreover, they serve as a significant theoretical motivation, besides to the intuitive ones, of the practical algorithms that we describe in the following sections.

\section{Sharing representations}
We want to empirically evaluate the benefit of our theoretical study in the problem of jointly learning $T$ different tasks $\mu_t$, introducing a neural network architecture for which our bounds hold. Following our theoretical framework, the network we propose extracts representations $w_t$ from inputs $x_t$ for each task $\mu_t$, mapping them to common features in a set of shared layers $h$, specializing the learning of each task in respective separated layers $f_t$, and finally computing the output $y_t = (f_t \circ h \circ w_t)(x_t) = f_t(h(w_t(x_t)))$ (Figure~\ref{F:network}). The idea behind this architecture is not new in the literature. For instance, similar ideas have already been used in \gls{dqn} variants to improve exploration on the same task via bootstrapping~\citep{osband2016deep} and to perform \gls{mtrl}~\citep{liu2016decoding}.

\begin{figure}
\centering
 \subfigure[Shared network\label{F:network}]
 {\includegraphics[scale=.4]{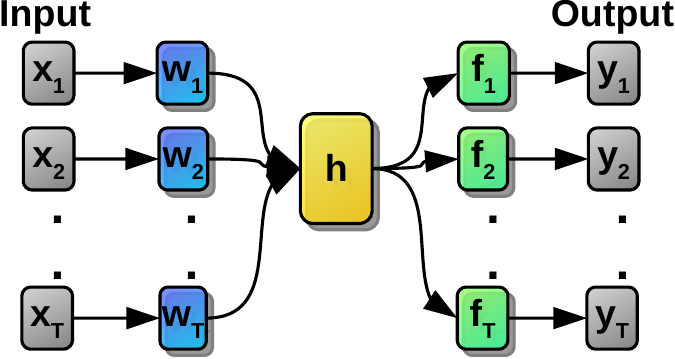}}\hspace{-.1cm}
 \subfigure[\gls{fqi} vs \gls{mfqi}\label{F:fqi}]{\includegraphics[scale=.22]{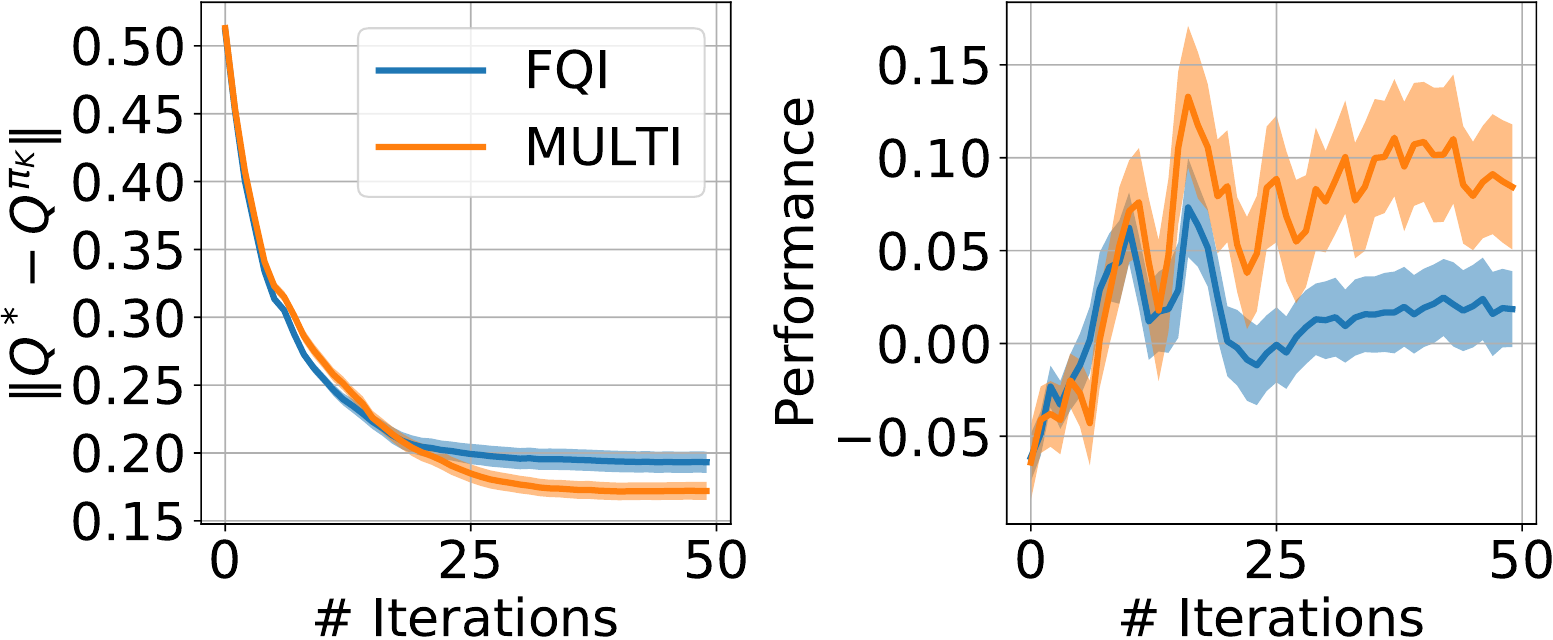}}
 \subfigure[\#Task analysis\label{F:multiple-tasks}]{\includegraphics[scale=.2]{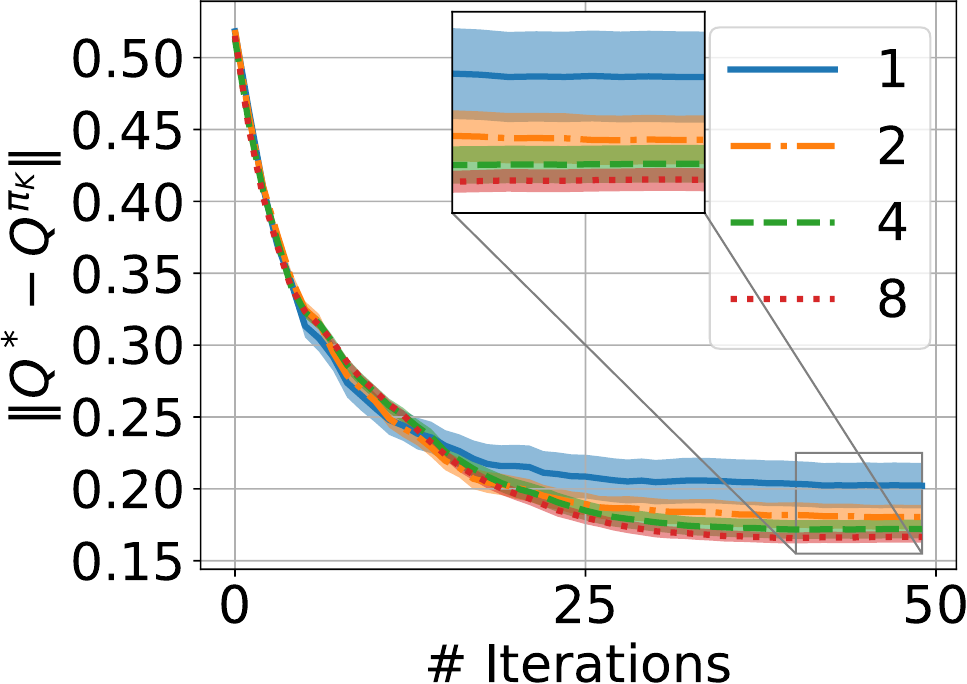}}
 \caption{(a) The architecture of the neural network we propose to learn $T$ tasks simultaneously. The $w_t$ block maps each input $x_t$ from task $\mu_t$ to a \textit{shared} set of layers $h$ which extracts a common representation of the tasks. Eventually, the shared representation is specialized in block $f_t$ and the output $y_t$ of the network is computed. Note that each block can be composed of arbitrarily many layers. (b) Results of \gls{fqi} and \gls{mfqi} averaged over $4$ tasks in \textit{Car-On-Hill}, showing $\lVert Q^* - Q^{\pi_K}\rVert$ on the left, and the discounted cumulative reward on the right. (c) Results of \gls{mfqi} showing $\lVert Q^* - Q^{\pi_K}\rVert$ for increasing number of tasks. Both results in (b) and (c) are averaged over $100$ experiments, and show the $95\%$ confidence intervals.}
\end{figure}

The intuitive and desirable property of this architecture is the exploitation of the regularization effect introduced by the shared representation of the jointly learned tasks. Indeed, unlike learning a single task that may end up in overfitting, forcing the model to compute a shared representation of the tasks helps the regression process to extract more general features, with a consequent reduction in the variance of the learned function. This intuitive justification for our approach, joins the theoretical benefit proven in Section~\ref{S:bounds}. Note that our architecture can be used in any \gls{mtrl} problem involving a regression process; indeed, it can be easily used in value-based methods as a $Q$-function regressor, or in policy search as a policy regressor. In both cases, the targets are learned for each task $\mu_t$ in its respective output block $f_t$. Remarkably, as we show in the experimental Section~\ref{S:exp}, it is straightforward to extend \gls{rl} algorithms to their multi-task variants only through the use of the proposed network architecture, without major changes to the algorithms themselves.

\section{Experimental results}\label{S:exp}
To empirically evince the effect described by our bounds, we propose an extension of \gls{fqi}~\citep{ernst2005tree, riedmiller2005neural}, that we call \gls{mfqi}, for which our AVI bounds apply.
Then, to empirically evaluate our approach in challenging \gls{rl} problems, we introduce multi-task variants of two well-known \gls{drl} algorithms: \gls{dqn}~\citep{mnih2015human} and \gls{ddpg}~\citep{lillicrap2015continuous}, which we call \gls{mdqn} and \gls{mddpg} respectively. Note that for these methodologies, our \gls{avi} and \gls{api} bounds hold only with the simplifying assumption that the samples are i.i.d.; nevertheless they are useful to show the benefit of our method also in complex scenarios, e.g. MuJoCo~\citep{todorov2012mujoco}. We remark that in these experiments we are only interested in showing the benefit of learning multiple tasks with a shared representation w.r.t. learning a single task; therefore, we only compare our methods with the single task counterparts, ignoring other works on \gls{mtrl} in literature. Experiments have been developed using the MushroomRL library~\citep{deramo2020mushroomrl}, and run on an NVIDIA® DGX Station™ and Intel® AI DevCloud. Refer to Appendix~\ref{App:exp} for all the details and our motivations about the experimental settings.

\subsection{Multi Fitted $Q$-Iteration}
As a first empirical evaluation, we consider \gls{fqi}, as an example of an AVI algorithm, to show the effect described by our theoretical AVI bounds in experiments. We consider the \textit{Car-On-Hill} problem as described in~\cite{ernst2005tree}, and select four different tasks from it changing the mass of the car and the value of the actions (details in Appendix~\ref{App:exp}). Then, we run separate instances of FQI with a single task network for each task respectively, and one of \gls{mfqi} considering all the tasks simultaneously. Figure~\ref{F:fqi} shows the $L_1$-norm of the difference between $Q^*$ and $Q^{\pi_K}$ averaged over all the tasks. It is clear how \gls{mfqi} is able to get much closer to the optimal $Q$-function, thus giving an empirical evidence of the AVI bounds in Theorem~\ref{T:avi_bound}. For completeness, we also show the advantage of \gls{mfqi} w.r.t. \gls{fqi} in performance. Then, in Figure~\ref{F:multiple-tasks} we provide an empirical evidence of the benefit of increasing the number of tasks in \gls{mfqi} in terms of both quality and stability.

\subsection{Multi Deep $Q$-Network}
As in~\cite{liu2016decoding}, our \gls{mdqn} uses separate replay memories for each task and the batch used in each training step is built picking the same number of samples from each replay memory. Furthermore, a step of the algorithm consists of exactly one step in each task. These are the only minor changes to the vanilla \gls{dqn} algorithm we introduce, while all other aspects, such as the use of the target network, are not modified. Thus, the time complexity of \gls{mdqn} is considerably lower than vanilla \gls{dqn} thanks to the learning of $T$ tasks with a single model, but at the cost of a higher memory complexity for the collection of samples for each task.
We consider five problems with similar state spaces, sparse rewards and discrete actions: \textit{Cart-Pole}, \textit{Acrobot}, \textit{Mountain-Car}, \textit{Car-On-Hill}, and \textit{Inverted-Pendulum}. The implementation of the first three problems is the one provided by the OpenAI Gym library~\cite{openai-gym}, while Car-On-Hill is described in~\cite{ernst2005tree} and Inverted-Pendulum in~\cite{lagoudakis2003least}.

\begin{figure}
 \subfigure[Multi-task\label{F:dqn-singlemulti}]{\includegraphics[scale=.263]{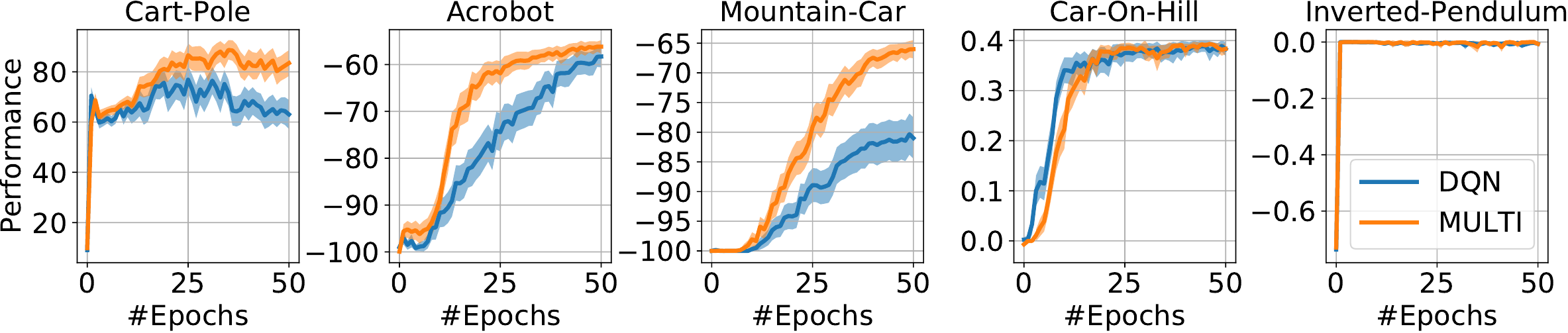}}
 \subfigure[Transfer\label{F:dqn-transfer}]{\includegraphics[scale=.15]{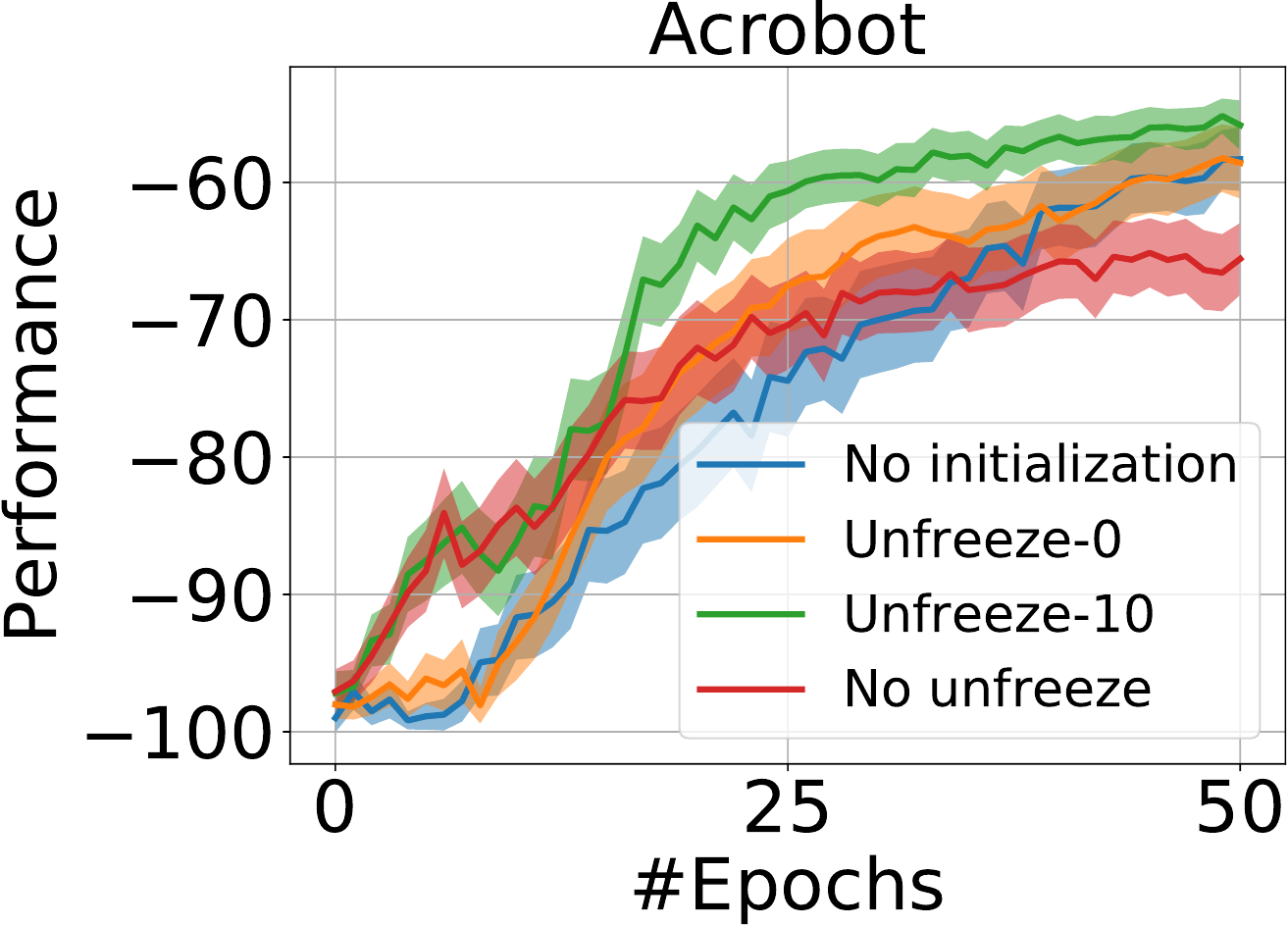}}
  \caption{Discounted cumulative reward averaged over $100$ experiments of \gls{dqn} and \gls{mdqn} for each task and for transfer learning in the \textit{Acrobot} problem. An epoch consists of $1,000$ steps, after which the greedy policy is evaluated for $2,000$ steps. The $95\%$ confidence intervals are shown.}
\end{figure}

Figure~\ref{F:dqn-singlemulti} shows the performance of \gls{mdqn} w.r.t. to vanilla \gls{dqn} that uses a single-task network structured as the multi-task one in the case with $T=1$. The first three plots from the left show good performance of \gls{mdqn}, which is both higher and more stable than \gls{dqn}. In Car-On-Hill, \gls{mdqn} is slightly slower than \gls{dqn} to reach the best performance, but eventually manages to be more stable. Finally, the Inverted-Pendulum experiment is clearly too easy to solve for both approaches, but it is still useful for the shared feature extraction in \gls{mdqn}. The described results provide important hints about the better quality of the features extracted by \gls{mdqn} w.r.t. \gls{dqn}. To further demonstrate this, we evaluate the performance of \gls{dqn} on Acrobot, arguably the hardest of the five problems, using a single-task network with the shared parameters in $h$ initialized with the weights of a multi-task network trained with \gls{mdqn} on the other four problems. Arbitrarily, the pre-trained weights can be adjusted during the learning of the new task or can be kept fixed and only the remaining randomly initialized parameters in $\mathbf{w}$ and $\mathbf{f}$ are trained. From Figure~\ref{F:dqn-transfer}, the advantages of initializing the weights are clear. In particular, we compare the performance of \gls{dqn} without initialization w.r.t. \gls{dqn} with initialization in three settings: in \textit{Unfreeze-0} the initialized weights are adjusted, in \textit{No-Unfreeze} they are kept fixed, and in \textit{Unfreeze-10} they are kept fixed until epoch $10$ after which they start to be optimized. Interestingly, keeping the shared weights fixed shows a significant performance improvement in the earliest epochs, but ceases to improve soon. On the other hand, the adjustment of weights from the earliest epochs shows improvements only compared to the uninitialized network in the intermediate stages of learning. The best results are achieved by starting to adjust the shared weights after epoch $10$, which is approximately the point at which the improvement given by the fixed initialization starts to lessen.

\begin{figure}
\centering
 \subfigure[Multi-task for pendulums\label{F:ddpg-singlemulti-pend}]{\includegraphics[scale=.275]{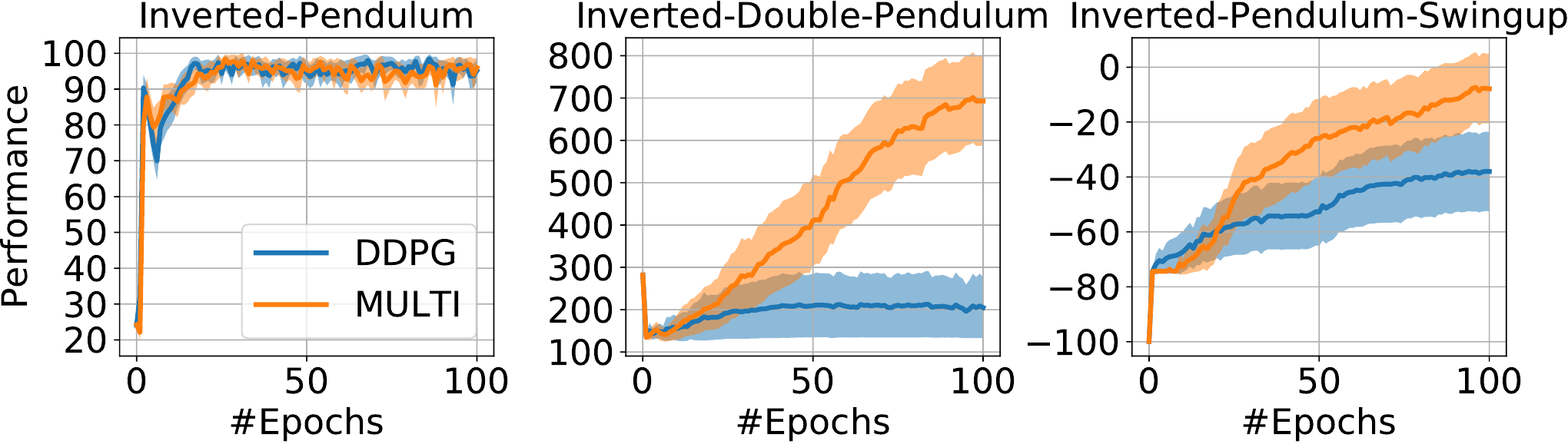}}
 \subfigure[Transfer for pendulums\label{F:ddpg-transfer-pend}]{\includegraphics[scale=.15]{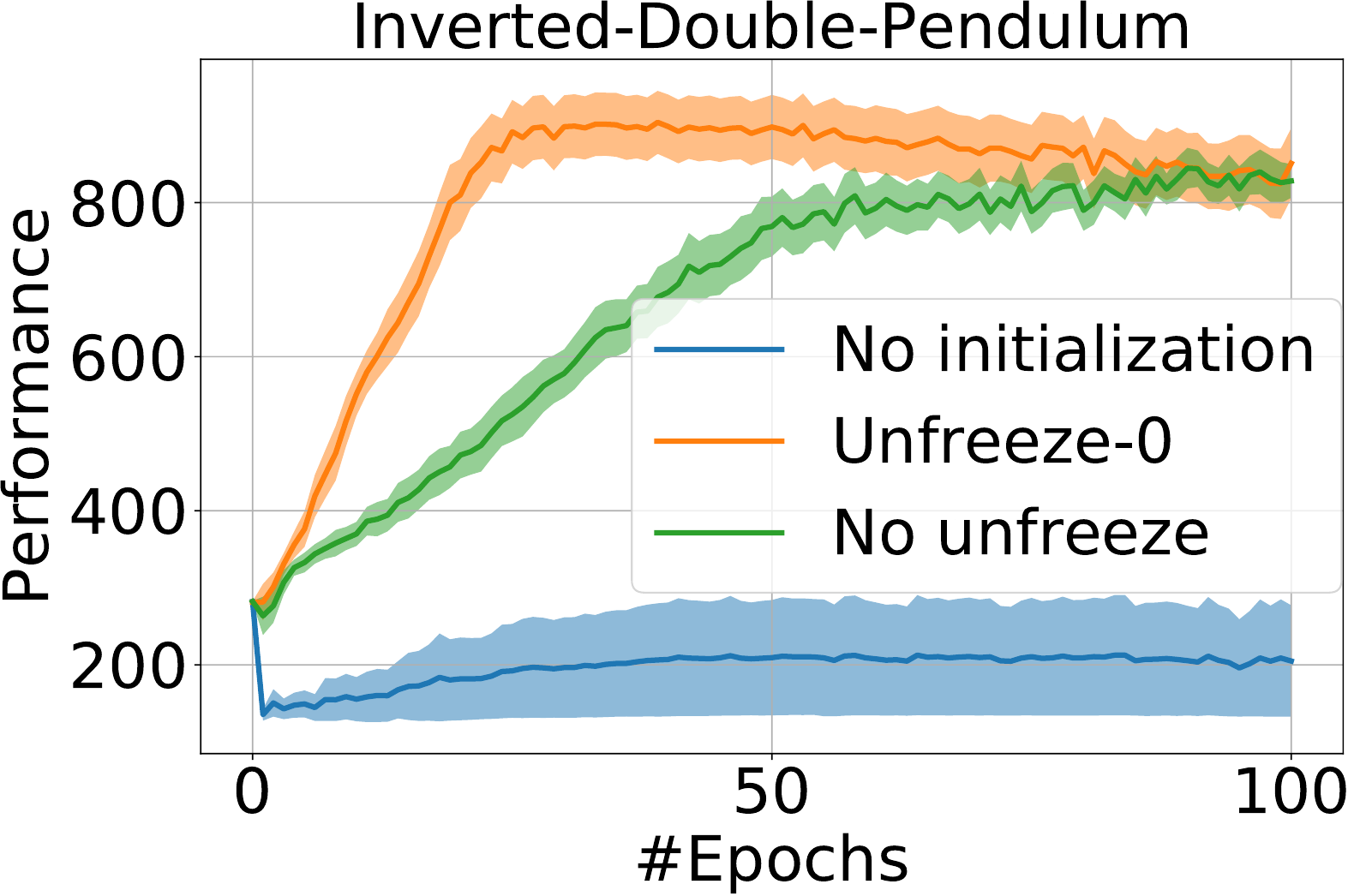}}
 \subfigure[Multi-task for walkers\label{F:ddpg-singlemulti-walk}]{\includegraphics[scale=.275]{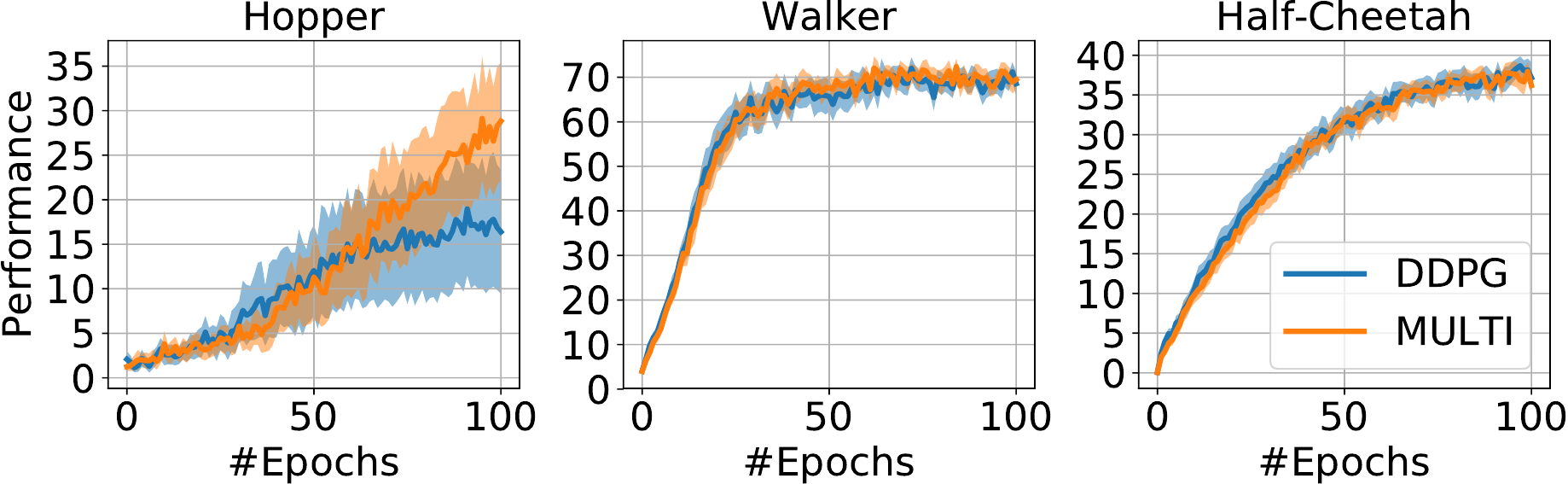}}\hspace{.5cm}
 \subfigure[Transfer for walkers\label{F:ddpg-transfer-walk}]{\includegraphics[scale=.15]{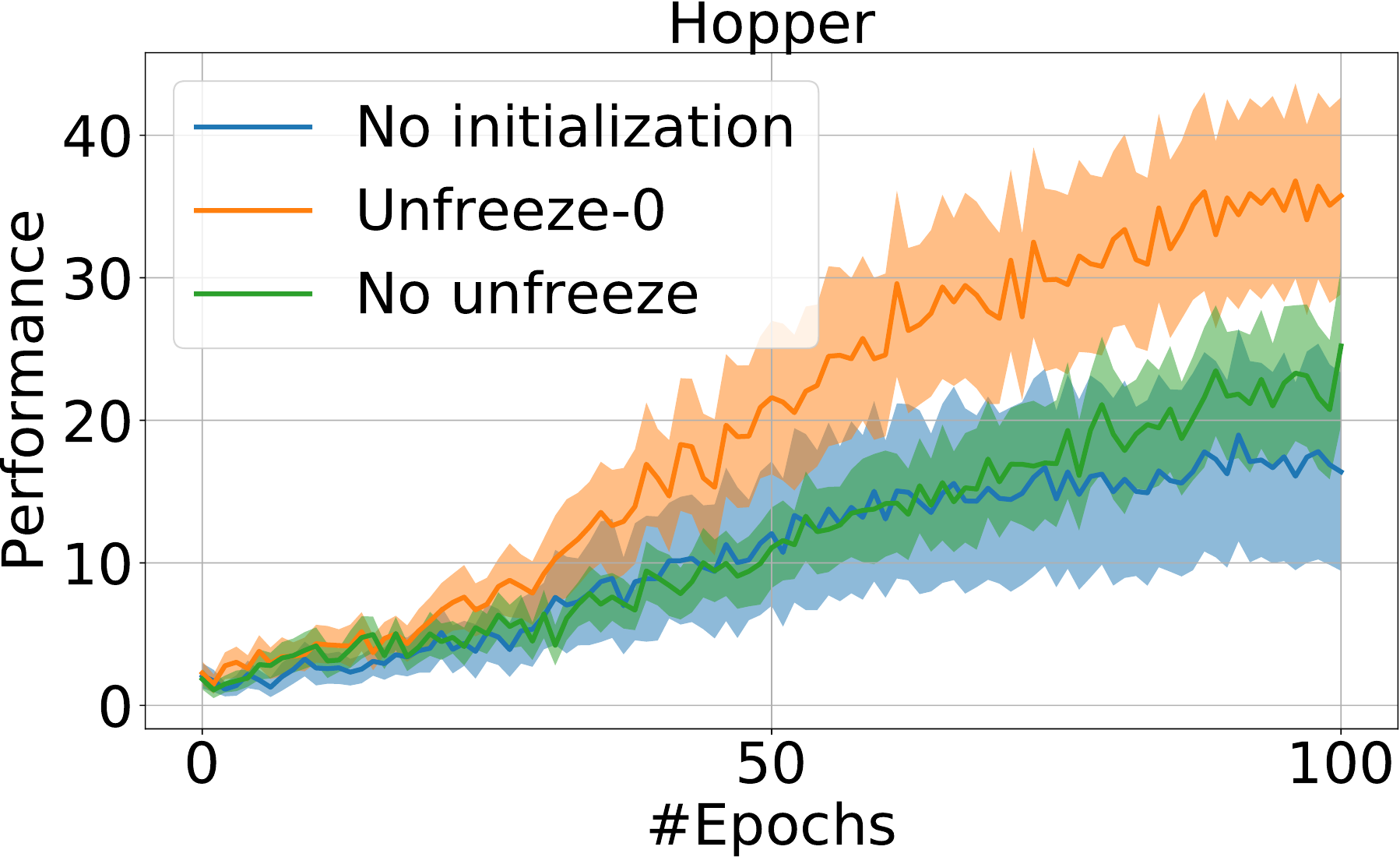}}
  \caption{Discounted cumulative reward averaged over $40$ experiments of \gls{ddpg} and \gls{mddpg} for each task and for transfer learning in the \textit{Inverted-Double-Pendulum} and \textit{Hopper} problems. An epoch consists of $10,000$ steps, after which the greedy policy is evaluated for $5,000$ steps. The $95\%$ confidence intervals are shown.}
\end{figure}
\subsection{Multi Deep Deterministic Policy Gradient}
In order to show how the flexibility of our approach easily allows to perform \gls{mtrl} in policy search algorithms, we propose \gls{mddpg} as a multi-task variant of \gls{ddpg}. As an actor-critic method, \gls{ddpg} requires an \textit{actor} network and a \textit{critic} network. Intuitively, to obtain \gls{mddpg} both the actor and critic networks should be built following our proposed structure. We perform separate experiments on two sets of MuJoCo~\cite{todorov2012mujoco} problems with similar continuous state and action spaces: the first set includes \textit{Inverted-Pendulum}, \textit{Inverted-Double-Pendulum}, and \textit{Inverted-Pendulum-Swingup} as implemented in the \textit{pybullet} library, whereas the second set includes \textit{Hopper-Stand}, \textit{Walker-Walk}, and \textit{Half-Cheetah-Run} as implemented in the DeepMind Control Suite\cite{DBLP:journals/corr/abs-1801-00690}. Figure~\ref{F:ddpg-singlemulti-pend} shows a relevant improvement of \gls{mddpg} w.r.t. \gls{ddpg} in the pendulum tasks. Indeed, while in Inverted-Pendulum, which is the easiest problem among the three, the performance of \gls{mddpg} is only slightly better than \gls{ddpg}, the difference in the other two problems is significant. The advantage of \gls{mddpg} is confirmed in Figure~\ref{F:ddpg-singlemulti-walk} where it performs better than \gls{ddpg} in Hopper and equally good in the other two tasks. Again, we perform a \gls{tl} evaluation of \gls{ddpg} in the problems where it suffers the most, by initializing the shared weights of a single-task network with the ones of a multi-task network trained with \gls{mddpg} on the other problems. Figures~\ref{F:ddpg-transfer-pend} and~\ref{F:ddpg-transfer-walk} show evident advantages of pre-training the shared weights and a significant difference between keeping them fixed or not.

\section{Related works}
Our work is inspired from both theoretical and empirical studies in \gls{mtl} and \gls{mtrl} literature. In particular, the theoretical analysis we provide follows previous results about the theoretical properties of multi-task algorithms. For instance,~\cite{cavallanti2010linear} and~\cite{maurer2006bounds} prove the theoretical advantages of \gls{mtl} based on linear approximation. More in detail,~\cite{maurer2006bounds} derives bounds on \gls{mtl} when a linear approximator is used to extract a shared representation among tasks. Then,~\cite{maurer2016benefit}, which we considered in this work, describes similar results that extend to the use of non-linear approximators. Similar studies have been conducted in the context of \gls{mtrl}. Among the others,~\cite{lazaric2011transfer} and~\cite{brunskill2013sample} give theoretical proofs of the advantage of learning from multiple \glspl{mdp} and introduces new algorithms to empirically support their claims, as done in this work.

Generally, contributions in \gls{mtrl} assume that properties of different tasks, e.g. dynamics and reward function, are generated from a common generative model. About this, interesting analyses consider Bayesian approaches; for instance~\cite{wilson2007multi} assumes that the tasks are generated from a hierarchical Bayesian model, and likewise~\cite{lazaric2010bayesian} considers the case when the value functions are generated from a common prior distribution. Similar considerations, which however does not use a Bayesian approach, are implicitly made in~\cite{taylor2007transfer},~\cite{lazaric2008transfer}, and also in this work.

In recent years, the advantages of \gls{mtrl} have been empirically evinced also in \gls{drl}, especially exploiting the powerful representational capacity of deep neural networks. For instance,~\cite{parisotto2015actor} and~\cite{rusu2015policy} propose to derive a multi-task policy from the policies learned by \gls{dqn} experts trained separately on different tasks.~\cite{rusu2015policy} compares to a therein introduced variant of \gls{dqn}, which is very similar to our \gls{mdqn} and the one in~\cite{liu2016decoding}, showing how their method overcomes it in the Atari benchmark~\cite{bellemare2013arcade}. Further developments, extend the analysis to policy search~\citep{yang2017multi, teh2017distral}, and to multi-goal \gls{rl}~\citep{schaul2015universal,andrychowicz2017hindsight}. Finally,~\cite{hessel2018multi} addresses the problem of balancing the learning of multiple tasks with a single deep neural network proposing a method that uniformly adapts the impact of each task on the training updates of the agent.

\section{Conclusion}
We have theoretically proved the advantage in \gls{rl} of using a shared representation to learn multiple tasks w.r.t. learning a single task. We have derived our results extending the \gls{avi}/\gls{api} bounds~\citep{farahmand2011regularization} to \gls{mtrl}, leveraging the upper bounds on the approximation error in \gls{mtl} provided in~\cite{maurer2016benefit}. The results of this analysis show that the error propagation during the \gls{avi}/\gls{api} iterations is reduced according to the number of tasks. Then, we proposed a practical way of exploiting this theoretical benefit which consists in an effective way of extracting shared representations of multiple tasks by means of deep neural networks. To empirically show the advantages of our method, we carried out experiments on challenging \gls{rl} problems with the introduction of multi-task extensions of \gls{fqi}, \gls{dqn}, and \gls{ddpg} based on the neural network structure we proposed. As desired, the favorable empirical results confirm the theoretical benefit we described.

\newpage
\subsubsection*{Acknowledgments}
This project has received funding from the European Union’s Horizon 2020 research and innovation programme under grant agreement No. \#640554 (SKILLS4ROBOTS) and No. \#713010 (GOAL-Robots). This project has also been supported by grants from NVIDIA, the NVIDIA DGX Station, and the Intel\textsuperscript{\textregistered} AI DevCloud. The authors thank Alberto Maria Metelli, Andrea Tirinzoni and Matteo Papini for their helpful insights during the development of the project.
\bibliography{bibliography}
\bibliographystyle{bibliography}

\newpage
\onecolumn
\appendix

\section{Proofs}\label{App:proofs}
\subsection{Approximated Value-Iteration bounds}

\begin{proof}[Proof of Theorem \ref{T:avi_bound}]
 We compute the average expected loss across tasks:
\begin{align}
 \dfrac{1}{T} &\sum_{t=1}^T \lVert Q^*_t - Q_t^{\pi_K}\rVert_{1,\rho}\nonumber\\
 & \leq \dfrac{1}{T} \sum_{t=1}^T \dfrac{2\gamma_t}{(1-\gamma_t)^2}\left[\inf_{r \in [0,1]} C^{\frac{1}{2}}_{\text{VI},\rho,\nu}(K;t,r) \mathcal{E}^{\frac{1}{2}}(\varepsilon_{t,0}, \dots, \varepsilon_{t,K-1};t,r)+\dfrac{2}{1 - \gamma_t}\gamma^K_t R_{\text{max},t}\right]\nonumber\\
 & \leq \dfrac{2\gamma}{(1-\gamma)^2} \dfrac{1}{T} \sum_{t=1}^T \left[\inf_{r \in [0,1]} C^{\frac{1}{2}}_{\text{VI},\rho,\nu}(K;t,r)\mathcal{E}^{\frac{1}{2}}(\varepsilon_{t,0}, \dots, \varepsilon_{t,K-1};t,r)+\dfrac{2}{1 - \gamma_t}\gamma^K_t R_{\text{max},t}\right]\nonumber\\
 & \leq \dfrac{2\gamma}{(1-\gamma)^2} \left[ \dfrac{1}{T} \sum_{t=1}^T \left(\inf_{r \in [0,1]} C^{\frac{1}{2}}_{\text{VI},\rho,\nu}(K;t,r)\mathcal{E}^{\frac{1}{2}}(\varepsilon_{t,0}, \dots, \varepsilon_{t,K-1};t,r)\right) + \dfrac{2}{1 - \gamma}\gamma^K R_{\text{max},\text{avg}}\right]\nonumber\\
 & \leq \dfrac{2\gamma}{(1-\gamma)^2} \left[ \inf_{r \in [0,1]} \dfrac{1}{T} \sum_{t=1}^T \left(C^{\frac{1}{2}}_{\text{VI},\rho,\nu}(K;t,r)\mathcal{E}^{\frac{1}{2}}(\varepsilon_{t,0}, \dots, \varepsilon_{t,K-1};t,r)\right) + \dfrac{2}{1 - \gamma}\gamma^K R_{\text{max},\text{avg}}\right]\nonumber\\
 & \leq \dfrac{2\gamma}{(1-\gamma)^2} \left[ \inf_{r \in [0,1]} C_{\text{VI}}^{\frac{1}{2}}(K;r) \dfrac{1}{T} \sum_{t=1}^T \left(\mathcal{E}^{\frac{1}{2}}(\varepsilon_{t,0}, \dots, \varepsilon_{t,K-1};t,r)\right) + \dfrac{2}{1 - \gamma}\gamma^K R_{\text{max},\text{avg}}\right]\label{E:avi_multi_farahmand}
\end{align}
with $\gamma = \underset{t \in \lbrace 1,\dots,T \rbrace}{\max} \gamma_t$, $C_{\text{VI}}^{\frac{1}{2}}(K;r) = \underset{t \in \lbrace 1, \dots, T \rbrace}{\max} C^{\frac{1}{2}}_{\text{VI},\rho,\nu}(K;t,r)$, and $R_{\text{max},\text{avg}} = \nicefrac{1}{T} \sum_{t=1}^T R_{\text{max},t}$.\\
Considering the term $\nicefrac{1}{T} \sum_{t=1}^T \left[\mathcal{E}^{\frac{1}{2}}(\varepsilon_{t,0}, \dots, \varepsilon_{t,K-1};t,r)\right] = \nicefrac{1}{T} \sum_{t=1}^T \left(\sum^{K-1}_{k=0}\alpha_{t,k}^{2r}\varepsilon_{t,k}\right)^{\frac{1}{2}}$ let
$$\alpha_k = \begin{cases} \frac{(1 - \gamma)\gamma^{K-k-1}}{1-\gamma^{K+1}} & 0 \leq k < K,\\ \frac{(1-\gamma)\gamma^K}{1 - \gamma^{K+1}} & k = K \end{cases},$$
then we bound
\begin{equation*}
 \dfrac{1}{T} \sum_{t=1}^T \left(\sum^{K-1}_{k=0}\alpha_{t,k}^{2r}\varepsilon_{t,k}\right)^{\frac{1}{2}} \leq \dfrac{1}{T} \sum_{t=1}^T \left(\sum^{K-1}_{k=0}\alpha_{k}^{2r}\varepsilon_{t,k}\right)^{\frac{1}{2}}.
\end{equation*}

Using Jensen's inequality:
\begin{equation*}
 \dfrac{1}{T} \sum_{t=1}^T \left(\sum^{K-1}_{k=0}\alpha_{k}^{2r}\varepsilon_{t,k}\right)^{\frac{1}{2}} \leq \left(\sum^{K-1}_{k=0}\alpha_{k}^{2r}\dfrac{1}{T}\sum_{t=1}^T\varepsilon_{t,k}\right)^{\frac{1}{2}}.
\end{equation*}
So, now we can write (\ref{E:avi_multi_farahmand}) as
\begin{align*}
  \dfrac{1}{T} \sum_{t=1}^T \lVert Q^*_t - Q_t^{\pi_K}\rVert_{1,\rho} \leq \dfrac{2\gamma}{(1-\gamma)^2} &\left[ \inf_{r \in [0,1]} C_{\text{VI}}^{\frac{1}{2}}(K;r)\mathcal{E}_{\text{avg}}^{\frac{1}{2}}(\varepsilon_{\text{avg},0}, \dots, \varepsilon_{\text{avg},K-1};r)\right.\nonumber \\
  &\left.+ \dfrac{2}{1 - \gamma} \gamma^K R_{\text{max},\text{avg}}\right],
\end{align*}
with $\varepsilon_{\text{avg},k} = \nicefrac{1}{T}\sum_{t=1}^T \varepsilon_{\text{t}, k}$ and $\mathcal{E}_{\text{avg}}(\varepsilon_{\text{avg},0}, \dots, \varepsilon_{\text{avg},K-1};r) = \sum^{K-1}_{k=0}\alpha_k^{2r}\varepsilon_{\text{avg},k}$.

\end{proof}

\begin{proof}[Proof of Lemma~\ref{T:eps_star}]
Let us start from the definition of optimal task-averaged risk:
\begin{equation*}
 \varepsilon_{\text{avg},k}^*= \dfrac{1}{T}\sum_{t=1}^{T}\lVert Q_{t,k+1}^* - \mathcal{T}^*_tQ_{t,k} \rVert^2_\nu,
\end{equation*}
where $Q_{t,k}^*$, with $t\in[1, T]$, are the minimizers of $\varepsilon_{\text{avg},k}$.

Consider the task $\check{t}$ such that
\begin{equation*}
\check{t}_k = \arg \max_{t \in \lbrace 1, \dots, T \rbrace} \lVert Q_{t,k+1}^* - \mathcal{T}^*_tQ_{t,k} \rVert^2_\nu,
\end{equation*}
we can write the following inequality:
\begin{equation*}
 \sqrt{\varepsilon_{\text{avg},k}^*} \leq \lVert Q_{\check{t}_k,k+1}^* - \mathcal{T}^*_{\check{t}}Q_{\check{t}_k,k} \rVert_\nu.
\end{equation*}
By the application of Theorem 5.3 by~\cite{farahmand2011regularization} to the right hand side, and defining $b_{k,i} = \lVert Q_{\check{t}_k,i+1} - \mathcal{T}^*_{\check{t}}Q_{\check{t}_k,i}\rVert_\nu$, we obtain:
\begin{equation*}
 \sqrt{\varepsilon_{\text{avg},k}^*} \leq \lVert Q_{\check{t}_k,k+1}^* - (\mathcal{T}^*_{\check{t}})^{k+1}Q_{\check{t}_k,0} \rVert_\nu + \sum_{i=0}^{k-1}(\gamma_{\check{t}_k} C_{\text{AE}}(\nu;\check{t}_k, P))^{i+1}b_{k,k-1-i}.
\end{equation*}

Squaring both sides yields the result:
\begin{align*}
\varepsilon_{\text{avg},k}^* \leq \left(\vphantom{\sum_{i=0}^{k-1}}\lVert Q_{\check{t}_k,k+1}^* - (\mathcal{T}^*_{\check{t}})^{k+1}Q_{\check{t}_k,0} \rVert_\nu + \sum_{i=0}^{k-1}(\gamma_{\check{t}_k} C_{\text{AE}}(\nu;\check{t}_k, P))^{i+1}b_{k,k-1-i}\right)^2.
\end{align*}
\end{proof}

\subsection{Approximated Policy-Iteration bounds}\label{A:api_bound}
We start by considering the bound for the \gls{api} framework:
\begin{theorem}(Theorem 3.2 of~\cite{farahmand2011regularization})
  Let K be a positive integer, and $Q_{\text{max}} \leq \frac{R_{\text{max}}}{1-\gamma}$. Then for any sequence $(Q_k )^{K-1}_{k=0}\subset B(\mathcal{S}\times\mathcal{A}, Q_{\text{max}})$ and the corresponding sequence $(\varepsilon_k)_{k=0}^{K-1}$, where $\varepsilon_k=\lVert Q_{k} - Q^{\pi_k} \rVert^2_\nu$, we have:
  \begin{align}
    \lVert Q^* - Q^{\pi_K}\rVert_{1,\rho} \leq \dfrac{2\gamma}{(1-\gamma)^2}\left[\inf_{r \in [0,1]} C^{\frac{1}{2}}_{\text{PI},\rho,\nu}(K;r)\mathcal{E}^{\frac{1}{2}}(\varepsilon_{0}, \dots, \varepsilon_{K-1};r)+\gamma^{K-1} R_{\text{max}}\right],\label{E:farahmand_api}
  \end{align}
  where
  \begin{flalign}
     &C_{\text{PI},\rho,\nu}(K;r) = &\nonumber\\
     &\left(\dfrac{1-\gamma}{2}\right)^2\sup_{\pi'_0,\dots,\pi'_K}\sum^{K-1}_{k=0}a_k^{2(1-r)} & &\left(\sum_{m \geq 0}\gamma^m c_{\text{PI}_1,\rho,\nu}(K-k-1,m+1;\pi'_{k+1})+\right.&\nonumber\\ 
     & & &\left.\sum_{m\geq 1}\gamma^m c_{\text{PI}_2,\rho,\nu}(K-k-1,m;\pi'_{k+1},\pi'_k)\vphantom{\sum_{m \geq 0}}+c_{\text{PI}_3,\rho,\nu}\right)^2;&
  \end{flalign}
  with $\mathcal{E}(\varepsilon_{0}, \dots, \varepsilon_{K-1};r)=\sum^{K-1}_{k=0}\alpha_k^{2r}\varepsilon_{k}$, the three coefficients $c_{\text{PI}_1,\rho,\nu}$, $c_{\text{PI}_2,\rho,\nu}$, $c_{\text{PI}_3,\rho,\nu}$, the distributions $\rho$ and $\nu$, and the series $\alpha_k$ are defined as in~\cite{farahmand2011regularization}\label{T:farahmand_api}.
\end{theorem}

From Theorem~\ref{T:farahmand_api}, by computing the average loss across tasks and pushing inside the average using Jensen's inequality, we derive the \gls{api} bounds averaged on multiple tasks.

\begin{theorem}
Let K be a positive integer, and $Q_{\text{max}} \leq \frac{R_{\text{max}}}{1-\gamma}$. Then for any sequence $(Q_k )^{K-1}_{k=0}\subset B(\mathcal{S}\times\mathcal{A}, Q_{\text{max}})$ and the corresponding sequence $(\varepsilon_{\text{avg},k})_{k=0}^{K-1}$, where $\varepsilon_{\text{avg},k}=\dfrac{1}{T}\sum_{t=1}^T\lVert Q_{t,k} - Q_{t}^{\pi_k} \rVert^2_\nu$, we have:
 \begin{align}
   \dfrac{1}{T}\sum_{t=1}^T \lVert Q^*_t - Q_t^{\pi_K}\rVert_{1,\rho} \leq \dfrac{2\gamma}{(1-\gamma)^2} &\left[ \inf_{r \in [0,1]} C_{\text{PI}}^{\frac{1}{2}}(K;r) \mathcal{E}_{\text{avg}}^{\frac{1}{2}}(\varepsilon_{\text{avg},0}, \dots, \varepsilon_{\text{avg},K-1};r)\right.\nonumber\\
   &\left.+ \gamma^{K-1} R_{\text{max},\text{avg}}\right],\label{E:api_bound}
\end{align}
with $\mathcal{E}_{\text{avg}} = \sum^{K-1}_{k=0}\alpha_k^{2r}\varepsilon_{\text{avg},k}$, $\gamma = \underset{t \in \lbrace 1,\dots,T \rbrace}{\max} \gamma_t$, $C_{\text{PI}}^{\frac{1}{2}}(K;r) = \underset{t \in \lbrace 1, \dots, T \rbrace}{\max} C^{\frac{1}{2}}_{\text{PI},\rho,\nu}(K;t,r)$, $R_{\text{max},\text{avg}} = \dfrac{1}{T} \sum_{t=1}^T R_{\text{max},t}$ and $\alpha_k = \begin{cases} \frac{(1 - \gamma)\gamma^{K-k-1}}{1-\gamma^{K+1}} & 0 \leq k < K,\\ \frac{(1-\gamma)\gamma^K}{1 - \gamma^{K+1}} & k = K \end{cases}$.
\label{T:api_bound}
\end{theorem}

\begin{proof}[Proof of Theorem \ref{T:api_bound}]
The proof is very similar to the one for \gls{avi}. We compute the average expected loss across tasks:
\begin{align}
 \dfrac{1}{T} &\sum_{t=1}^T \lVert Q^*_t - Q_t^{\pi_K}\rVert_{1,\rho}\nonumber\\
 & \leq \dfrac{1}{T} \sum_{t=1}^T \dfrac{2\gamma_t}{(1-\gamma_t)^2}\left[\inf_{r \in [0,1]} C^{\frac{1}{2}}_{\text{PI},\rho,\nu}(K;t,r) \mathcal{E}^{\frac{1}{2}}(\varepsilon_{t,0}, \dots, \varepsilon_{t,K-1};t,r)+\gamma^{K-1}_t R_{\text{max},t}\right]\nonumber\\
 & \leq \dfrac{2\gamma}{(1-\gamma)^2} \dfrac{1}{T} \sum_{t=1}^T \left[\inf_{r \in [0,1]} C^{\frac{1}{2}}_{\text{PI},\rho,\nu}(K;t,r)\mathcal{E}^{\frac{1}{2}}(\varepsilon_{t,0}, \dots, \varepsilon_{t,K-1};t,r)+\gamma^{K-1}_t R_{\text{max},t}\right]\nonumber\\
 & \leq \dfrac{2\gamma}{(1-\gamma)^2} \left[ \dfrac{1}{T} \sum_{t=1}^T \left(\inf_{r \in [0,1]} C^{\frac{1}{2}}_{\text{PI},\rho,\nu}(K;t,r)\mathcal{E}^{\frac{1}{2}}(\varepsilon_{t,0}, \dots, \varepsilon_{t,K-1};t,r)\right) +\gamma^{K-1} R_{\text{max},\text{avg}}\right]\nonumber\\
 & \leq \dfrac{2\gamma}{(1-\gamma)^2} \left[ \inf_{r \in [0,1]} \dfrac{1}{T} \sum_{t=1}^T \left(C^{\frac{1}{2}}_{\text{PI},\rho,\nu}(K;t,r)\mathcal{E}^{\frac{1}{2}}(\varepsilon_{t,0}, \dots, \varepsilon_{t,K-1};t,r)\right) + \gamma^{K-1} R_{\text{max},\text{avg}}\right]\nonumber\\
 & \leq \dfrac{2\gamma}{(1-\gamma)^2} \left[ \inf_{r \in [0,1]} C_{\text{PI}}^{\frac{1}{2}}(K;r) \dfrac{1}{T} \sum_{t=1}^T \left(\mathcal{E}^{\frac{1}{2}}(\varepsilon_{t,0}, \dots, \varepsilon_{t,K-1};t,r)\right) + \gamma^{K-1} R_{\text{max},\text{avg}}\right].\label{E:api_multi_farahmand}
\end{align}

Using Jensen's inequality as in the \gls{avi} scenario, we can write~(\ref{E:api_multi_farahmand}) as:

\begin{align}
 \dfrac{1}{T}\sum_{t=1}^T \lVert Q^*_t - Q_t^{\pi_K}\rVert_{1,\rho} \leq \dfrac{2\gamma}{(1-\gamma)^2} &\left[ \inf_{r \in [0,1]} C_{\text{PI}}^{\frac{1}{2}}(K;r) \mathcal{E}_{\text{avg}}^{\frac{1}{2}}(\varepsilon_{\text{avg},0}, \dots, \varepsilon_{\text{avg},K-1};r)\right.\nonumber\\
 &\left.+ \gamma^{K-1} R_{\text{max},\text{avg}}\right],
\end{align}
with $\varepsilon_{\text{avg},k} = \nicefrac{1}{T}\sum_{t=1}^T \varepsilon_{\text{t}, k}$ and $\mathcal{E}_{\text{avg}}(\varepsilon_{\text{avg},0}, \dots, \varepsilon_{\text{avg},K-1};r) = \sum^{K-1}_{k=0}\alpha_k^{2r}\varepsilon_{\text{avg},k}$.
\end{proof}

\subsection{Approximation bounds}
\begin{proof}[Proof of Theorem \ref{T:apprx}]
Let $w_1^*, \dots, w_T^*$, $h^*$ and $f_1^*,\dots,f_T^*$ be the minimizers of $\varepsilon^*_{\text{avg}}$,  then:
\begin{align}
\varepsilon_{\text{avg}}(\mathbf{\hat{w}}, \hat{h}, \mathbf{\hat{f}}) - \varepsilon^*_{\text{avg}} &= \underbrace{\left( \varepsilon_{\text{avg}}(\mathbf{\hat{w}}, \hat{h}, \mathbf{\hat{f}}) - \frac{1}{nT}\sum_{ti}\ell(\hat{f}_t(\hat{h}(\hat{w}_t(X_{ti}))), Y_{ti})\right)}_{A}\nonumber\\
&+\underbrace{\left( \frac{1}{nT}\sum_{ti}\ell(\hat{f}_t(\hat{h}(\hat{w}_t(X_{ti}))), Y_{ti}) - \frac{1}{nT}\sum_{ti}\ell(f^*_t(h^*(w^*_t(X_{ti}))), Y_{ti})\right)}_{B}\nonumber\\
&+\underbrace{\left( \frac{1}{nT}\sum_{ti}\ell(f^*_t(h^*(w^*_t(X_{ti}))), Y_{ti}) - \varepsilon^*_{\text{avg}}\right)}_{C}.\label{E:abc_bound}
\end{align}
We proceed to bound the three components individually:
\begin{itemize}
 \item $C$ can be bounded using Hoeffding's inequality, with probability $1 - \nicefrac{\delta}{2}$ by $\sqrt{\nicefrac{\ln\left(\nicefrac{2}{\delta}\right)}{(2nT)}}$, as it contains only $nT$ random variables bounded in the interval $[0,1]$;
 \item $B$ can be bounded by $0$, by definition of $\mathbf{\hat{w}}$, $\hat{h}$ and $\mathbf{\hat{f}}$, as they are the minimizers of Equation~(\ref{E:mtrl}); 
 \item the bounding of $A$ is less straightforward and is described in the following.
\end{itemize}

We define the following auxiliary function spaces:
\begin{itemize}
 \item $\mathcal{W}'=\lbrace x \in \mathcal{X} \to (w_t(x_{ti})) : (w_1, \dots, w_T) \in \mathcal{W}^T \rbrace$,
 \item $\mathcal{F}'=\left\lbrace y \in \mathbb{R}^{KTn} \to (f_t(y_{ti})) : (f_1, \dots, f_T) \in \mathcal{F}^T \right\rbrace$,
\end{itemize}

and the following auxiliary sets:
\begin{itemize}
 \item $S=\left\{(\ell(f_t(h(w_t(X_{ti}))), Y_{ti})):f \in \mathcal{F}^T, h \in \mathcal{H}, w \in \mathcal{W}^T \right\} \subseteq \mathbb{R}^{Tn}$,
 \item $S' = \mathcal{F}'(\mathcal{H}(\mathcal{W}'(\bar{\mathbf{X}}))) = \left\{(f_t(h(w_t(X_{ti})))):f \in \mathcal{F}^T, h \in \mathcal{H}, w \in \mathcal{W}^T \right\} \subseteq \mathbb{R}^{Tn}$,
 \item $S'' = \mathcal{H}(\mathcal{W}'(\bar{\mathbf{X}})) = \left\{(h(w_t(X_{ti}))): h \in \mathcal{H}, w \in \mathcal{W}^T \right\}\subseteq \mathbb{R}^{KTn}$,
\end{itemize}
which will be useful in our proof.

Using Theorem $9$ by~\cite{maurer2016benefit}, we can write:
\begin{align}
\varepsilon_{\text{avg}}(\mathbf{\hat{w}}, \hat{h}, \mathbf{\hat{f}}) - \frac{1}{nT}\sum_{ti}&\ell(\hat{f}_t(\hat{h}(\hat{w}_t(X_{ti}))), Y_{ti})\nonumber\\ &\leq \sup_{\mathbf{w} \in \mathcal{W}^T, h \in \mathcal{H}, \mathbf{f} \in \mathcal{F}^T} \left( \varepsilon_{\text{avg}}(\mathbf{w}, h, \mathbf{f}) - \frac{1}{nT}\sum_{ti}\ell(f_t(h(w_t(X_{ti}))), Y_{ti})\right)\nonumber\\
&\leq \frac{\sqrt{2\pi}G(S)}{nT} + \sqrt{\frac{9\ln(\frac{2}{\delta})}{2nT}}\label{E:bound_1},
\end{align}
then by Lipschitz property of the loss function $\ell$ and the contraction lemma Corollary $11$~\cite{maurer2016benefit}: $G(S) \leq G(S')$. By Theorem $12$ by~\cite{maurer2016benefit}, for universal constants $c_1'$ and $c_2'$:
\begin{equation}
G(S') \leq c_1'L(\mathcal{F}')G(S'') + c_2'D(S'')O(\mathcal{F}') + \min_{y \in Y} G(\mathcal{F}(y)),\label{E:G_S}
\end{equation}
where $L(\mathcal{F}')$ is the largest value for the Lipschitz constants in the function space $\mathcal{F}'$, and $D(S'')$ is the Euclidean diameter of the set $S''$.

Using Theorem $12$ by~\cite{maurer2016benefit} again, for universal constants $c_1''$ and $c_2''$:
\begin{equation}
G(S'') \leq c_1''L(\mathcal{H})G(\mathcal{W}'(\bar{\mathbf{X}})) + c_2''D(\mathcal{W}'(\bar{\mathbf{X}}))O(\mathcal{H}) + \min_{p \in P}G(\mathcal{H}(p)).\label{E:G_S'}
\end{equation}
Putting~(\ref{E:G_S}) and~(\ref{E:G_S'}) together:
\begin{align}
 G(S') \leq &\,c_1'L(\mathcal{F}')\left(c_1''L(\mathcal{H})G(\mathcal{W}'(\bar{\mathbf{X}})) + c_2''D(\mathcal{W}'(\bar{\mathbf{X}}))O(\mathcal{H}) + \min_{p \in P}G(\mathcal{H}(p))\right)\nonumber\\
 &+ c_2'D(S'')O(\mathcal{F}') + \min_{y \in Y} G(\mathcal{F}(y))\nonumber\\
 =&\,c_1'c_1''L(\mathcal{F}')L(\mathcal{H})G(\mathcal{W}'(\bar{\textbf{X}})) + c_1'c_2''L(\mathcal{F}')D(\mathcal{W}'(\bar{\textbf{X}}))O(\mathcal{H}) + c_1'L(\mathcal{F}')\min_{p \in P}G(\mathcal{H}(p))\nonumber\\
 &+ c_2'D(S'')O(\mathcal{F}') + \min_{y \in Y} G(\mathcal{F}(y)).\label{E:G_chain}
\end{align}
At this point, we have to bound the individual terms in the right hand side of~(\ref{E:G_chain}), following the same procedure proposed by~\cite{maurer2016benefit}. 

Firstly, to bound $L(\mathcal{F}')$, let $y,y'\in\mathbb{R}^{KTn}$, where $y=(y_{ti})$ with $y_{ti}\in\mathbb{R}^K$ and $y'=(y'_{ti})$ with $y'_{ti}\in\mathbb{R}^K$. We can write the following:
\begin{align}
 \lVert f(y) - f(y') \rVert^2 & = \sum_{ti}\left(f_t(y_{ti})-f_t(y'_{ti})\right)^2 \nonumber\\
 & \leq L(\mathcal{F})^2\sum_{ti}\lVert y_{ti}-y'_{ti} \rVert^2 \nonumber\\
 & = L(\mathcal{F})^2\lVert y-y' \rVert^2,
\end{align}
whence $L(\mathcal{F}') \leq L(\mathcal{F})$.

Then, we bound:
\begin{align}
 G(\mathcal{W}'(\bar{\mathbf{X}}))=\mathbb{E}\left[\sup_{\mathbf{w}\in\mathcal{W}^T}\sum_{kti}\gamma_{kti}w_{tk}(X_{ti}) \middle| X_{ti} \right] &\leq \sum_{t}\sup_{l\in\{1,\dots,T\}}\mathbb{E}\left[\sup_{w\in\mathcal{W}}\sum_{ki}\gamma_{kli}w_{k}(X_{li}) \middle| X_{li}\right] \nonumber\\
 &= T\sup_{l\in\{1,\dots,T\}} G(\mathcal{W}(\mathbf{X}_l)).\label{E:G_w}
\end{align}

Then, since it is possible to bound the Euclidean diameter using the norm of the supremum value in the set, we bound $D(S'') \leq 2 \sup_{h,\mathbf{w}}\lVert h(\mathbf{w}(\bar{\mathbf{X}}))\rVert$ and $D(\mathcal{W}'(\bar{\mathbf{X}})) \leq 2\sup_{\mathbf{w}\in\mathcal{W}^T}\lVert\mathbf{w}(\bar{\mathbf{X}})\rVert$.

Also, we bound $O(\mathcal{F}')$:
\begin{align}
\mathbb{E}\left[\sup_{g \in \mathcal{F}'}\langle\boldsymbol{\gamma}, g(y) - g(y')\rangle\right] & =  \mathbb{E}\left[ \sup_{\boldsymbol{f} \in \mathcal{F}^T}\sum_{ti}\gamma_{ti} \left(f_t(y_{ti})-f_t(y_{ti}')\right) \right] \nonumber\\
& = \sum_{t} \mathbb{E}\left[ \sup_{f \in \mathcal{F}}\sum_{i}\gamma_{i} \left(f(y_{ti})-f(y_{ti}')\right) \right] \nonumber\\
& \leq \sqrt{T} \left( \sum_{t} \mathbb{E}\left[ \sup_{f \in \mathcal{F}}\sum_{i}\gamma_{i} \left(f(y_{ti})-f(y_{ti}')\right) \right]^2 \right)^{\frac{1}{2}} \nonumber\\
& \leq \sqrt{T} \left( \sum_{t} O(\mathcal{F})^2 \sum_i \lVert y_{ti} - y_{ti}' \rVert^2 \right)^{\frac{1}{2}} \nonumber\\
& = \sqrt{T} O(\mathcal{F})\lVert y - y' \rVert,
\end{align}
whence $O(\mathcal{F}')\leq \sqrt{T} O(\mathcal{F})$.

To minimize the last term, it is possible to choose $y_0=0$, as $f(0)=0, \forall f\in\mathcal{F}$, resulting in $\min_{y \in Y} G(\mathcal{F}(y)) = G(\mathcal{F}(0)) = 0$.

Then, substituting in~(\ref{E:G_chain}), and recalling that $G(S)\leq G(S')$:
\begin{align}
 G(S) \leq &\,c_1'c_1''L(\mathcal{F})L(\mathcal{H})T\sup_{l\in\{1,\dots,T\}} G(\mathcal{W}(\mathbf{X}_l)) + 2c_1'c_2''L(\mathcal{F}) \sup_{\mathbf{w} \in \mathcal{W}^T} \lVert \mathbf{w}(\bar{\mathbf{X}})\rVert O(\mathcal{H})\nonumber\\
 &+ c_1'L(\mathcal{F})\min_{p \in P}G(\mathcal{H}(p))+ 2c_2' \sup_{h,\mathbf{w}}\lVert h(\mathbf{w}(\bar{\mathbf{X}}))\rVert\sqrt{T}O(\mathcal{F}).\label{E:gs_bound}
\end{align}
Now, the first term $A$ of~(\ref{E:abc_bound}) can be bounded substituting~(\ref{E:gs_bound}) in~(\ref{E:bound_1}):
\begin{align*}
\varepsilon_{\text{avg}}(\mathbf{\hat{w}}, \hat{h}, \mathbf{\hat{f}}) &- \frac{1}{nT}\sum_{ti}\ell(\hat{f}_t(\hat{h}(\hat{w}_t(X_{ti}))), Y_{ti})\nonumber\\
\leq& \dfrac{\sqrt{2\pi}}{nT}\Big( c_1'c_1''L(\mathcal{F})L(\mathcal{H})T\sup_{l\in\{1,\dots,T\}} G(\mathcal{W}(\mathbf{X}_l)) + 2c_1'c_2''L(\mathcal{F}) \sup_{\mathbf{w} \in \mathcal{W}^T}\lVert \mathbf{w}(\bar{\mathbf{X}})\rVert O(\mathcal{H})\nonumber\\
&+ c_1'L(\mathcal{F})\min_{p \in P}G(\mathcal{H}(p))+ 2c_2' \sup_{h,\mathbf{w}}\lVert h(\mathbf{w}(\bar{\mathbf{X}}))\rVert\sqrt{T}O(\mathcal{F})\Big) + \sqrt{\frac{9\ln(\frac{2}{\delta})}{2nT}}\nonumber\\
=&\,c_1\dfrac{L(\mathcal{F})L(\mathcal{H})\sup_{l\in\{1,\dots,T\}} G(\mathcal{W}(\mathbf{X}_l))}{n} + c_2\dfrac{\sup_\mathbf{w}\lVert \mathbf{w}(\bar{\mathbf{X}})\rVert L(\mathcal{F}) O(\mathcal{H})}{nT}\nonumber\\
&+ c_3\dfrac{L(\mathcal{F})\min_{p \in P}G(\mathcal{H}(p))}{nT} + c_4 \dfrac{\sup_{h,\mathbf{w}}\lVert h(\mathbf{w}(\bar{\mathbf{X}}))\rVert O(\mathcal{F})}{n\sqrt{T}} + \sqrt{\frac{9\ln(\frac{2}{\delta})}{2nT}}.
\end{align*}
A union bound between $A$, $B$ and $C$ of~(\ref{E:abc_bound}) completes the proof:
\begin{align*}
\varepsilon_{\text{avg}}(\mathbf{\hat{w}}, \hat{h}, \mathbf{\hat{f}}) - \varepsilon^*_{\text{avg}} \leq& 
 \,c_1\dfrac{L(\mathcal{F})L(\mathcal{H})\sup_{l\in\{1,\dots,T\}} G(\mathcal{W}(\mathbf{X}_l))}{n}\nonumber\\
&+ c_2\dfrac{\sup_\mathbf{w}\lVert \mathbf{w}(\bar{\mathbf{X}})\rVert L(\mathcal{F}) O(\mathcal{H})}{nT}\nonumber\\
&+ c_3\dfrac{L(\mathcal{F})\min_{p \in P}G(\mathcal{H}(p))}{nT}\nonumber\\
&+ c_4 \dfrac{\sup_{h,\mathbf{w}}\lVert h(\mathbf{w}(\bar{\mathbf{X}}))\rVert O(\mathcal{F})}{n\sqrt{T}}\nonumber\\
&+ \sqrt{\frac{8\ln(\frac{3}{\delta})}{nT}}.
\end{align*}

\end{proof}

\section{Additional details of empirical evaluation}\label{App:exp}

\subsection{Multi Fitted $Q$-Iteration}
We consider \textit{Car-On-Hill} problem with discount factor $0.95$ and horizon $100$. Running Adam optimizer with learning rate $0.001$ and using a mean squared loss, we train a neural network composed of $2$ shared layers of $30$ neurons each, with sigmoidal activation function, as described in~\cite{riedmiller2005neural}. We select $8$ tasks for the problem changing the mass of the car $m$ and the value of the discrete actions $a$ (Table~\ref{Ta:fqi_tasks}). Figure~\ref{F:fqi} is computed considering the first four tasks, while Figure~\ref{F:multiple-tasks} considers task $1$ in the result with $1$ task, tasks $1$ and $2$ for the result with $2$ tasks, tasks $1$, $2$, $3$, and $4$ for the result with $4$ tasks, and all the tasks for the result with $8$ tasks.
To run \gls{fqi} and \gls{mfqi}, for each task we collect transitions running an extra-tree trained following the procedure and setting in~\cite{ernst2005tree}, using an $\epsilon$-greedy policy with $\epsilon=0.1$, to obtain a small, but representative dataset. The optimal $Q$-function for each task is computed by tree-search\footnote{We follow the method described in~\cite{ernst2005tree}.} for $100$ states uniformly picked from the state space, and the $2$ discrete actions, for a total of $200$ state-action tuples.

\subsection{Multi Deep $Q$-Network}
The five problems we consider for this experiment are: \textit{Cart-Pole}, \textit{Acrobot}, \textit{Mountain-Car}, \textit{Car-On-Hill}, and \textit{Inverted-Pendulum}\footnote{The IDs of the problems in the OpenAI Gym library are: \textit{CartPole-v0}, \textit{Acrobot-v1}, and \textit{MountainCar-v0}.}. The discount factors are respectively $0.99$, $0.99$, $0.99$, $0.95$, and $0.95$. The horizons are respectively $500$, $1,000$, $1,000$, $100$, and $3,000$. The network we use consists of $80$ ReLu units for each $w_t, t \in \lbrace 1, \dots, T \rbrace$ block, with $T=5$. Then, the shared block $h$ consists of one layer with $80$ ReLu units and another one with $80$ sigmoid units. Eventually, each $f_t$ has a number of linear units equal to the number of discrete actions $a^{(t)}_{i}, i \in \lbrace 1, \dots, \# \mathcal{A}^{(t)} \rbrace$ of task $\mu_t$ which outputs the action-value $Q_t(s, a^{(t)}_i) = y_t(s, a^{(t)}_i) = f_t(h(w_t(s)), a^{(t)}_i), \forall s \in \mathcal{S}^{(t)}$. The use of sigmoid units in the second layer of $h$ is due to our choice to extract meaningful shared features bounded between $0$ and $1$ to be used as input of the last linear layer, as in most \gls{rl} approaches. In practice, we have also found that sigmoid units help to reduce task interference in multi-task networks, where instead the linear response of ReLu units cause a problematic increase in the feature values. Furthermore, the use of a bounded feature space reduces the $\sup_{h,\mathbf{w}}\lVert h(\mathbf{w}(\bar{\mathbf{X}})) \rVert$ term in the upper bound of Theorem~\ref{T:apprx}, corresponding to the upper bound of the diameter of the feature space, as shown in Appendix~\ref{App:proofs}. The initial replay memory size for each task is $100$ and the maximum size is $5,000$. We use Huber loss with Adam optimizer using learning rate $10^{-3}$ and batch size of $100$ samples for each task. The target network is updated every $100$ steps. The exploration is $\varepsilon$-greedy with $\varepsilon$ linearly decaying from $1$ to $0.01$ in the first $5,000$ steps.

\begin{table}
\begin{center}
\begin{tabular}{|c|c|c|}
\hline
Task & Mass & Action set\\
\hline
1 & 1.0 & $\lbrace -4.0;4.0\rbrace$\\
\hline
2 & 0.8 & $\lbrace -4.0;4.0\rbrace$\\
\hline
3 & 1.0 & $\lbrace -4.5;4.5\rbrace$\\
\hline
4 & 1.2 & $\lbrace -4.5;4.5\rbrace$\\
\hline
5 & 1.0 & $\lbrace -4.125;4.125\rbrace$\\
\hline
6 & 1.0 & $\lbrace -4.25;4.25\rbrace$\\
\hline
7 & 0.8 & $\lbrace -4.375;4.375\rbrace$\\
\hline
8 & 0.85 & $\lbrace -4.0;4.0\rbrace$\\
\hline
\end{tabular}
\caption{Different values of the mass of the car and available actions chosen for the Car-On-Hill tasks in the \gls{mfqi} empirical evaluation.}\label{Ta:fqi_tasks}
\end{center}
\end{table}

\subsection{Multi Deep Deterministic Policy Gradient}
The two set of problems we consider for this experiment are: one including \textit{Inverted-Pendulum}, \textit{Inverted-Double-Pendulum}, and \textit{Inverted-Pendulum-Swingup}, and another one including \textit{Hopper-Stand}, \textit{Walker-Walk}, and \textit{Half-Cheetah-Run}\footnote{The IDs of the problems in the pybullet library are: \textit{InvertedPendulumBulletEnv-v0}, \textit{InvertedDoublePendulumBulletEnv-v0}, and \textit{InvertedPendulumSwingupBulletEnv-v0}. The names of the domain and the task of the problems in the DeepMind Control Suite are: \textit{hopper-stand}, \textit{walker-walk}, and \textit{cheetah-run}.}. The discount factors are $0.99$ and the horizons are $1,000$ for all problems. The actor network is composed of $600$ ReLu units for each $w_t, t \in \lbrace 1, \dots, T \rbrace$ block, with $T=3$. The shared block $h$ has $500$ units with ReLu activation function as for \gls{mdqn}. Finally, each $f_t$ has a number of \textit{tanh} units equal to the number of dimensions of the continuous actions $a^{(t)} \in \mathcal{A}^{(t)}$ of task $\mu_t$ which outputs the policy $\pi_t(s) = y_t(s) = f_t(h(w_t(s))), \forall s \in \mathcal{S}^{(t)}$. On the other hand, the critic network consists of the same $w_t$ units of the actor, except for the use of sigmoidal units in the $h$ layer, as in \gls{mdqn}. In addition to this, the actions $a^{(t)}$ are given as input to $h$. Finally, each $f_t$ has a single linear unit $Q_t(s, a^{(t)}) = y_t(s, a^{(t)}) = f_t(h(w_t(s), a^{(t)})), \forall s \in \mathcal{S}^{(t)}$. The initial replay memory size for each task is $64$ and the maximum size is $50,000$. We use Huber loss to update the critic network and the policy gradient to update the actor network. In both cases the optimization is performed with Adam optimizer and batch size of $64$ samples for each task. The learning rate of the actor is $10^{-4}$ and the learning rate of the critic is $10^{-3}$. Moreover, we apply $\ell_2$-penalization to the critic network using a regularization coefficient of $0.01$. The target networks are updated with soft-updates using $\tau=10^{-3}$. The exploration is performed using the action computed by the actor network adding a noise generated with an Ornstein-Uhlenbeck process with $\theta = 0.15$ and $\sigma=0.2$. Note that most of these values are taken from the original \gls{ddpg} paper~\cite{lillicrap2015continuous}, which optimizes them for the single-task scenario.

\end{document}

%% file: math_commands.tex
\usepackage{amsmath,amsfonts,bm}

\def\eqref#1{equation~\ref{#1}}

\def\1{\bm{1}}

\DeclareMathAlphabet{\mathsfit}{\encodingdefault}{\sfdefault}{m}{sl}
\SetMathAlphabet{\mathsfit}{bold}{\encodingdefault}{\sfdefault}{bx}{n}



%% file: acronyms.tex
\newacronym{rl}{RL}{Reinforcement Learning}
\newacronym{drl}{DRL}{Deep Reinforcement Learning}
\newacronym{mtl}{MTL}{Multi-Task Learning}
\newacronym{mtrl}{MTRL}{Multi-Task Reinforcement Learning}
\newacronym{mtdrl}{MTDRL}{Multi-Task Deep Reinforcement Learning}
\newacronym{avi}{AVI}{Approximate Value-Iteration}
\newacronym{api}{API}{Approximate Policy-Iteration}
\newacronym{tl}{TL}{Transfer Learning}
\newacronym{ml}{ML}{Machine Learning}
\newacronym[plural=MDPs, firstplural=Markov Decision Processes (MDPs)]{mdp}{MDP}{Markov Decision Process}
\newacronym{dl}{DL}{Deep Learning}
\newacronym{dqn}{DQN}{Deep $Q$-Network}
\newacronym{ddpg}{DDPG}{Deep Deterministic Policy Gradient}
\newacronym{pac}{PAC}{Probably Approximately Correct}
\newacronym{mdqn}{MDQN}{Multi Deep $Q$-Network}
\newacronym{mddpg}{MDDPG}{Multi Deep Deterministic Policy Gradient}
\newacronym{fqi}{FQI}{Fitted $Q$-Iteration}
\newacronym{mfqi}{MFQI}{Multi Fitted $Q$-Iteration}